\documentclass{article}
\usepackage[algo2e,ruled,vlined,linesnumbered]{algorithm2e}
\usepackage{wrapfig}

\PassOptionsToPackage{sort,numbers, compress}{natbib}
 \usepackage[preprint]{cfrname_2026}
 \usepackage{amsmath}

\usepackage[utf8]{inputenc} 
\usepackage[T1]{fontenc}    
\usepackage{hyperref}       
\usepackage{url}            
\usepackage{xurl}
\usepackage{booktabs}       
\usepackage{amsfonts}       
\usepackage{nicefrac}       
\usepackage{microtype}      
\usepackage{xcolor}         

\usepackage{amsmath,amssymb,amsthm}
\usepackage{algorithm}
\usepackage{algpseudocode}
\usepackage{xfrac}
\usepackage{booktabs}
\usepackage{titletoc}
\usepackage{geometry}
\usepackage{tabularx}
\usepackage{graphicx}
\usepackage{cleveref}
\usepackage{makecell}
\usepackage{siunitx}
\usepackage{multicol}
\usepackage{multirow}
\usepackage{tikz}
\usepackage{mathrsfs}
\usepackage{xcolor}
\usepackage{subcaption}
\usepackage{setspace}
\usepackage{xfrac}
\usepackage{listings}

\crefname{lstlisting}{listing}{listings}
\Crefname{lstlisting}{Listing}{Listings}
\newcommand{\cmark}{\textcolor{green!60!black}{$\checkmark$}}
\newcommand{\xmark}{\textcolor{red!70!black}{$\times$}}
\newcommand{\pmark}{\textcolor{orange!80!black}{$\sim$}}

\definecolor{darkgreen}{RGB}{0,100,0}
\definecolor{darkred}{RGB}{139,0,0}

\theoremstyle{plain}
\newtheorem{theorem}{Theorem} 
\newtheorem{lemma}{Lemma} 
\newtheorem{corollary}{Corollary} 
\theoremstyle{definition}
\newtheorem{assumption}{Assumption} 
\theoremstyle{remark}

\newcommand{\Identity}{{\rm I\kern-.2em l}}
\newcommand{\Indicator}{{\rm I\kern-.2em l}}

\newcommand{\Expectbracket}[1]{\mathbb{E}\left[ #1 \right]}
\newcommand{\Expectsubbracket}[2]{\mathbb{E}_{#1}\left[ #2 \right]}
\newcommand{\Expectcond}[2]{\mathbb{E}\left[\left. #1 \right| #2 \right]}

\newcommand{\normsq}[1]{\left\Vert #1 \right\Vert^2}
\newcommand{\innerprod}[1]{\left\langle #1 \right\rangle}

\newcommand{\feedbackprob}{\phi}

\makeatletter
\renewcommand{\SetKwInOut}[2]{%
  \sbox\algocf@inoutbox{\KwSty{#2}\algocf@typo:}%
  \expandafter\ifx\csname InOutSizeDefined\endcsname\relax
    \newcommand\InOutSizeDefined{}\setlength{\inoutsize}{\wd\algocf@inoutbox}%
    \sbox\algocf@inoutbox{\parbox[t]{\inoutsize}{\KwSty{#2}\algocf@typo:\hfill}~}\setlength{\inoutindent}{\wd\algocf@inoutbox}%
  \else
    \ifdim\wd\algocf@inoutbox>\inoutsize%
    \setlength{\inoutsize}{\wd\algocf@inoutbox}%
    \sbox\algocf@inoutbox{\parbox[t]{\inoutsize}{\KwSty{#2}\algocf@typo:\hfill}~}\setlength{\inoutindent}{\wd\algocf@inoutbox}%
    \fi%
  \fi
  \algocf@newcommand{#1}[1]{%
    \ifthenelse{\boolean{algocf@inoutnumbered}}{\relax}{\everypar={\relax}}%
    {\let\\\algocf@newinout\hangindent=\inoutindent\hangafter=1\parbox[t]{\inoutsize}{\KwSty{#2}\algocf@typo:\hfill}~##1\par}%
    \algocf@linesnumbered
  }}%
\makeatother

\SetKwInput{KwInput}{Input}
\SetKwInput{KwOutput}{Output}

\newcommand{\algname}[0]{SLARouter\xspace}

\definecolor{pdnavy}{HTML}{0F1B35}
\definecolor{pdcardblue}{HTML}{1A3A6B}
\definecolor{pdpanelblue}{HTML}{2A5298}
\definecolor{pddarkpanel}{HTML}{0F2550}
\definecolor{pdencodergreen}{HTML}{0D4A42}
\definecolor{pdteal}{HTML}{0D9488}
\definecolor{pdorange}{HTML}{E07B39}
\definecolor{pdoffwhite}{HTML}{E8EDF5}
\definecolor{pdmuted}{HTML}{8FA3C0}
\definecolor{pdwhite}{HTML}{FFFFFF}
\definecolor{pdcardborder}{HTML}{3D6DB5}

\usepackage{enumitem}
\setlist[enumerate]{itemsep=0pt, topsep=0pt, parsep=3pt, partopsep=0pt, leftmargin=*}
\setlist[itemize]{itemsep=0pt, topsep=0pt, parsep=3pt, partopsep=0pt, leftmargin=*}

\allowdisplaybreaks

\title{Cost-Optimal LLM Routing with Limited User Feedback under User Satisfaction Guarantees}

\author{%
  Herbert Woisetschläger \\
  Technical University of Munich\\
  Germany \\
  \texttt{h.woisetschlaeger@tum.de} \\
  \And
  Arastun Mammadli \\
  University of Exeter \\
  United Kingdom \\
  \texttt{am2169@exeter.ac.uk} \\
  \And
  Ryan Zhang \\
  Horace Greeley High School \\
  United States \\
  \texttt{ryzhangofficial@gmail.com} \\
  \And
  Shiqiang Wang \\
  University of Exeter \\
  United Kingdom \\
  \texttt{s.wang9@exeter.ac.uk} \\
}

\begin{document}

\maketitle

\begin{abstract}
Inference costs for large language model (LLM) applications are rapidly growing, driven by surging demand and rising infrastructure cost. 
Users expect high-quality responses, and in commercial settings this is formally codified in Service Level Agreements (SLAs), creating a fundamental tension between cost and quality. 
Recent progress on cost-aware LLM request routing has shown potential to resolve this tension, but existing approaches rely on complete feedback signals, offline training, extensive per-workload tuning, and most lack SLA guarantees or inference-time adaptivity. 
We introduce SLARouter, an online routing algorithm that learns a cost-optimal policy from the sparse, one-sided user feedback available in production systems. 
SLARouter provides theoretical guarantees for both cost optimality and strict SLA compliance. Experiments across a wide range of LLM benchmarks show that SLARouter satisfies SLA constraints without the need for per-benchmark tuning, reducing operating cost by up to $2.2\times$ over existing baselines.

\end{abstract}

\section{Introduction}
 
The rapid proliferation of open-weight large language models (LLMs)~\cite{qwen3,granite4,glm5,minimax}, alongside a growing array of API-based offerings, has left both end users and inference service providers facing an increasingly intractable model selection problem.
Compounding this challenge is the fragmented benchmarking landscape. Results are often measured on incomparable tasks under inconsistent conditions, making it difficult to determine which model is best suited for a given workload~\cite{naveed2023overview}.
This leaves practitioners resorting to educated guessing or expensive trial-and-error.
 
Against this backdrop, \emph{LLM routing}, i.e., the problem of dynamically selecting an appropriate model for each incoming request, has emerged as a popular technique to improve user satisfaction and reduce inference costs~\cite{ong2024routellm,routerdc,tensoropera,llmdna,lookahead}.
Instead of committing to a single model, routing frameworks maintain a \emph{model zoo} and dispatch each request to whichever model best balances quality and cost for that particular query.
As model families often span parameter counts from under a billion to over a hundred billion, the cost difference between the smallest and largest model in a zoo can easily exceed an order of magnitude~\cite{samsi2023words}, making intelligent routing essential for sustainable operations.
 
Cost efficiency is especially pressing as LLMs are increasingly deployed in a wide range of applications, where sequences of model calls compound energy expenditure. 
Emerging AI regulation, such as Article~95 of the EU AI Act~\cite{euaiact}, further mandates energy monitoring and compliance with sustainable computing best practices.
Although there are existing studies on cost-aware routing and first steps towards cost-optimal routing, the following core challenges still remain.

\textbf{The Need for Formal Service-Level Guarantees.}
In commercial deployments, cost savings alone are insufficient. Providers and their customers typically codify quality requirements in \emph{Service Level Agreements} (SLAs).
For LLM inference, an SLA might specify a certain task success rate for an AI model over time.
Early routing work~\cite{ong2024routellm,ding2024hybridllm,aggarwal2024automix,chen2023frugalgpt} demonstrated that cost-aware routing between a strong and a weak model is practically effective, but do not provide formal guarantees that the routed system will meet a prescribed quality threshold.
More recent efforts have begun to close this gap by introducing minimax rate-optimal guarantees in a two-model setting~\cite{somerstep2025carrot} or formal SLA compliance to larger model zoos~\cite{bhatti2026proteus,woisetschlager2025messplus}, albeit most without adaptation at inference time.
Fully online methods that continuously update their routing policy as requests arrive remain sparse, especially those that simultaneously provide formal SLA satisfaction guarantees.
 
\textbf{The Feedback Bottleneck.}
A practical obstacle in virtually all existing routing approaches is that they assume the availability of \emph{reliable, timely, and complete} quality signals.
In real-world systems, however, explicit user feedback is sparse. 
For example, fewer than 5\% of requests may receive any signal~\cite{zhao2025metrics,feedback_sparsity}. 
When the feedback arrives, it is usually \textit{one-sided}, i.e., it pertains to exactly \emph{one} model response.
So, the performance of all other models remains unknown.
Without addressing this feedback bottleneck, even a theoretically sound routing policy will degrade in practice. 
We therefore need a routing framework that adapts to user preferences and workload shifts over time based on one-sided observational feedback while providing SLA guarantees.

\textbf{Research Question.} \textit{Can we build an LLM router that provides formal SLA guarantees and operates cost-optimally, relying only on sparse one-sided observational feedback available in practice?}

\begin{table*}[t]
    \centering
    \caption{Comparison of LLM request routing methods.
    \textit{Cost-Aware} means the method enables cost control.
    \textit{SLA Guar.}\ means the method provides formal SLA guarantees.
    \textit{Online} means the method can learn at inference time.
    \textit{Observ.\ Data} means the method can be trained from one-sided observational data where only the selected model's feedback is collected.}
    \label{tab:routing-comparison}
    \small
    \renewcommand{\arraystretch}{0.9}
    \resizebox{\textwidth}{!}{%
        \begin{tabular}{@{}lrccccl@{}}
        \toprule
        \textbf{Method} &
        \textbf{Zoo Size} & \textbf{Cost-Aware} & \textbf{SLA Guar.} &
        \textbf{Online} & \textbf{Observ.\ Data} & \textbf{Source} \\
        \midrule
        \multicolumn{7}{@{}l}{\textit{Cost-Aware without Optimality Guarantees}} \\
        RouteLLM    & $2$    & \cmark & \xmark & \xmark          & \xmark & \citet{ong2024routellm} \\
        Hybrid LLM  & $2$    & \cmark & \xmark & \xmark          & \xmark & \citet{ding2024hybridllm} \\
        AutoMix     & $2$    & \cmark & \xmark & \xmark          & \xmark & \citet{aggarwal2024automix} \\
        FrugalGPT   & $> 2$  & \cmark & \xmark & \xmark          & \xmark & \citet{chen2023frugalgpt} \\
        xRouter     & $> 2$  & \cmark & \xmark & \pmark$^\dagger$ & \xmark & \citet{qian2025xrouter} \\
        Causal LLM Routing & $> 2$ & \cmark & \xmark & \xmark   & \cmark & \citet{tsiourvas2025causal} \\
        \midrule
        \multicolumn{7}{@{}l}{\textit{Formal Optimality Guarantees Requiring Complete User Feedback}} \\
        CARROT      & $>2$    & \cmark & \pmark$^\ddagger$ & \xmark          & \xmark & \citet{somerstep2025carrot} \\
        PROTEUS     & $> 2$  & \cmark & \cmark            & \pmark$^\dagger$ & \xmark & \citet{bhatti2026proteus} \\
        MESS+       & $> 2$  & \cmark & \cmark            & \cmark          & \xmark & \citet{woisetschlager2025messplus} \\
        \midrule
        \multicolumn{7}{@{}l}{\textit{Formal Optimality Guarantees under One-Sided and Sparse User Feedback}} \\
        \textbf{\algname} & $> 2$ & \cmark & \cmark & \cmark & \cmark & \textbf{This Paper} \\
        \bottomrule
        \end{tabular}%
    }
    \medskip
    \begin{minipage}{\textwidth}
    \tiny\raggedright
    $^\dagger$\,Online RL training with offline deployment (no adaptation at inference time).
    $^\ddagger$\,Minimax rate-optimality guarantee, not an SLA satisfaction rate.
    \end{minipage}
    \vspace{-2em}
\end{table*}
 
\textbf{Contributions.}
We introduce \textbf{\algname}, a practical and scalable LLM routing algorithm that synthesizes the lessons of prior work into a unified framework capable of operating under production-realistic conditions. \algname makes the following novel contributions:
\begin{itemize}
    \item \textbf{Online routing from sparse one-sided observational feedback.} 
    \algname is an adaptive routing algorithm that learns a cost-optimal routing policy from one-sided and sparse user feedback as it is typically available in production systems. 
    We show that \algname is up to $2.2\times$ cheaper compared to existing cost-aware routing techniques while providing on par routing performance. 
    While our algorithm also proves to be effective in systems with low user feedback rates, we extend our routing system with an LLM-based judge mechanism~\cite{zheng2023judging} to ensure smooth predictor training. 

    \item \textbf{Formal SLA guarantees under observational data.} 
    We provide theoretical guarantees for cost-optimality under one-sided sparse user feedback while maintaining a required user satisfaction rate over time. 
    Our approach is based on a novel extension to the Lyapunov drift-plus-penalty framework~\cite{neely2022stochastic}. %
    With our performance analysis, we demonstrate that we can effectively learn from observational data and achieve virtual queue stability in settings with highly sparse feedback rates. 
\end{itemize}

\textbf{Related Work.}
\Cref{tab:routing-comparison} compares \algname to existing LLM routing algorithms.
Cost-aware routers such as RouteLLM~\cite{ong2024routellm}, Hybrid LLM~\cite{ding2024hybridllm}, AutoMix~\cite{aggarwal2024automix}, and FrugalGPT~\cite{chen2023frugalgpt} reduce inference expenditure effectively but offer no formal quality guarantees.
CARROT~\cite{somerstep2025carrot} introduces rate-optimal guarantees but does not adapt online.
PROTEUS~\cite{bhatti2026proteus} scales SLA compliance to larger model zoos but requires offline training, limiting adaptation to distributional shifts.
MESS+~\cite{woisetschlager2025messplus} combines online adaptation with formal SLA guarantees, but assumes complete, balanced, and timely feedback for all candidates that is rarely satisfied in production.
Causal LLM Routing~\cite{tsiourvas2025causal} is the only prior method that learns from observational data, but does not support online adaptation or SLA guarantees.
\algname is the first to satisfy all five criteria, specifically, cost-awareness, formal SLA guarantees, fully online operation, and learning from the one-sided sparse feedback available in real systems.

\section{\algname: Practical and Scalable LLM Request Routing with Service Level Guarantees based on Observational Data}

\subsection{Problem Formulation}

We consider a model zoo of $M$ LLMs, indexed by $m \in \{1, \ldots, M\}$, that serve incoming user requests $t \in \{1, \ldots, T\}$. 
For each request, our goal is to select a model that minimizes inference cost while guaranteeing that a minimum fraction $\alpha$ of requests receive a satisfactory response over time.

\textbf{Inference Cost (Objective).} 
Each model $m$ incurs an inference cost $C_{m,t}$ per request $t$, e.g., energy consumption in megajoules or monetary API cost. 
Model families routinely span orders of magnitude in cost, so routing intelligently instead of defaulting to the largest model can yield substantial savings.

\textbf{Service Level Agreement (Constraint).} 
In our SLA, we define $\alpha \in (0, 1)$ as the minimum required user satisfaction rate over time. 
Violating this threshold is typically associated with contractual penalties, making strict compliance non-negotiable. 
In practice, $\alpha$ should be set slightly above the contractually agreed threshold to absorb estimation noise and provide a compliance safety margin.

\textbf{Control Problem.} 
Combining objective and constraint, we aim to solve:
\begin{subequations}
\begin{align}
    \textstyle\min_{\{y_{m,t}:\forall t, m\}} \quad & \textstyle\frac{1}{T} \sum\nolimits_{t=1}^{T} \sum\nolimits_{m=1}^{M} \mathbb{E}\left[y_{m,t} C_{m,t}\right], \label{eq:objective} \\
    \text{s.t.} \quad & \textstyle\frac{1}{T} \sum\nolimits_{t=1}^{T} \sum\nolimits_{m=1}^{M} \mathbb{E}\left[y_{m,t} s_{m,t}\right] \geq \alpha, \label{eq:constraint} \\
    & \textstyle\sum\nolimits_{m=1}^{M} y_{m,t} = 1, \quad \forall t \in \{1, \ldots, T\}, \label{eq:one_model} \\
    &\textstyle y_{m,t} \in \{0, 1\}, \quad \forall t \in \{1, \ldots, T\},\ m \in \{1, \ldots, M\}, \label{eq:binary}
\end{align}
\label{eq:original-problem}%
\end{subequations}
where $y_{m,t} = 1$ if model $m$ serves request $t$ and $0$ otherwise, and $s_{m,t} \in \{0, 1\}$ denotes whether that response was satisfactory.

\textbf{Challenges.} 
Our work is primarily motivated by the observation of how users provide feedback in systems like ChatGPT, Claude, or Microsoft Copilot. 
Users have the option to give a thumbs up or down, whenever they would like to provide feedback. 
Overall feedback signals are extremely sparse~\cite{eisenstein2025metrics, feedback_sparsity}.
Since we usually only query one model for a user request, this typically leaves us with \emph{one-sided} and often \emph{imbalanced data}, i.e., we do not know the true value of $s_{m,t}$ for most requests and models.
Consequently, we have do deal with \emph{unknown performance statistics} for all other models that were not queried and did not receive feedback.
Furthermore, chat-like  \textit{requests arrive sequentially} in a streaming fashion without the knowledge of \emph{statistics of future requests.}
Hence, we cannot solve problem \eqref{eq:original-problem} directly in one shot over the entire time horizon. An \textit{online} solution that routes the model \textit{per request} is needed.

\subsection{Proposed \algname Algorithm}

The Lyapunov drift-plus-penalty framework~\cite{neely2022stochastic} gives a useful starting point for our problem. 
It operates in an online fashion without requiring knowledge of the underlying request or feedback distribution. 
It can provide provable guarantees on time-average performance despite unknown future statistics. 
This stability condition emerges organically from the Lyapunov structure instead being imposed as an ad hoc heuristic. However, the standard framework in~\cite{neely2022stochastic} as well as its use in~\cite{woisetschlager2025messplus} still assumes full feedback signals for all the requests. 
In this work, we present a \textit{novel extension} to support
sparse and imbalanced observations, ensuring that models with insufficient feedback are not perpetually neglected.
The overall procedure is in \Cref{algo:slarouter}. We explain the details as follows.

\textbf{Methodology.}
\algname operates in two phases, \emph{exploration} and \emph{exploitation}, controlled by a random process $X_t \sim \text{Bernoulli}(p_t)$, where $p_t = \min\left(1, \sfrac{c}{\sqrt[4]{t}}\right)$ and $c \in \mathbb{R}^+$, which is a hyperparameter to control the exploration probability. 
The exploration probability decays over time, gradually shifting weight toward exploitation as the predictor matures. Central to both phases is a \emph{satisfaction predictor}, which is a multi-label classifier with one output head per candidate model that independently estimates the probability (denoted by $\hat{s}_{m,t}$) that a given request $t$ would satisfy a user if served by model $m$. 
\begin{wrapfigure}{r}{0.52\textwidth}
    \vspace{-0.5em}
    \begin{algorithm2e}[H]
        \caption{\algname\xspace  -- Cost-optimal \& SLA compliant routing with sparse and one-sided user feedback} \label{algo:slarouter}
        \begin{spacing}{0.7}
        \scriptsize
        Initialize $\mathcal{D}\leftarrow\emptyset$, $Q_1\leftarrow 0$, predictor parameter $\mathbf{z}$\;
                
        \For{$t \leftarrow 1$ \KwTo $T$}{
            Compute $p_t \leftarrow \min\left(1, \sfrac{c}{\sqrt[4]{t}}\right)$\;
            Sample $X_t \sim \text{Bernoulli}(p_t)$\;
    
            \eIf{$X_t = 1$ or $t = 1$}{
                \tcp{Only sample one model}
                Sample $\tilde{m} \sim \text{Uniform}(\{1, 2, \ldots, M\})$\;
                $y_{\Tilde{m},t} \leftarrow 1$, $y_{m',t} \leftarrow 0, \forall m' \neq \tilde{m}$\;

                \eIf{$s_{\tilde{m},t} \neq \emptyset$}{
                    \tcp{User provides feedback}
                    $\mathcal{D}\leftarrow\mathcal{D}\cup(\mathbf{a}_t, s_{\tilde{m},t}, \tilde{m})$\;
                    Sample mini-batch $\mathcal{B}$ from $\mathcal{D}$\;
                    $\mathbf{z}_{t+1} \leftarrow \mathbf{z}_t - \eta_t \nabla F(\mathbf{z}_t, \mathcal{B})$\;
                }{
                    \tcp{User does not provide feedback}
                    $\mathbf{z}_{t+1} \leftarrow \mathbf{z}_t$\;
                }
            }{
                Predict request satisfaction probability $\hat{s}_{m,t}, \forall m$\;
                $\Tilde{m} \leftarrow \arg\min_m \left\{V C_{m,t} + Q_t(\alpha - \hat{s}_{m,t})\right\}$\;
                $y_{\Tilde{m},t} \leftarrow 1$, $y_{m',t} \leftarrow 0, \forall m' \neq \Tilde{m}$\;
                $\mathbf{z}_{t+1} \leftarrow \mathbf{z}_t$\;
            }

            \tcp{Virtual queue update}
            \eIf{$s_{\Tilde{m},t} \neq \emptyset$}{
                \tcp{Use user feedback for $Q$ update}
                $Q_{t+1} \leftarrow \max\{0, Q_t + \alpha - s_{\Tilde{m},t}\}$\;
            }{
                \tcp{Use predictor estimate for $Q$ update}
                $Q_{t+1} \leftarrow \max\{0, Q_t + \alpha - \hat{s}_{\Tilde{m},t}\}$\;
            }
        }
    \end{spacing}
    \end{algorithm2e}
    \vspace{-2em}
\end{wrapfigure}
Note, the predictor informs the routing decision but is not responsible for it. 
The decision is determined by our per-request optimization described below. 
Once a response was sent to the user, they can provide feedback to indicate whether it was satisfactory or not. 
As we aim to ensure SLA compliance, we use a \textit{virtual queue}  to capture all violations (negative feedback) over time. 
The queue length $Q_t$ affects the routing decision as more violations will pressure the algorithm to choose a more capable and expensive model, and vice versa.

\textbf{Feedback Availability.}
Whenever a request is sent to an LLM, the system may or may not receive user feedback about whether the response is satisfactory. This same principle applies to both the exploration and exploitation phases. We do not distinguish between true user feedback and feedback obtained from a separate LLM judge in this section. The performance difference between using user vs. LLM-judge feedback will be studied in the experiments later. If feedback is unavailable (from either the user or the LLM judge) for request $t$ and model $m$, we write $s_{m,t}=\emptyset$.

\textbf{Predictor.}
\label{par:predictor-training}
The predictor (with trainable parameter vector $\mathbf{z}$) must handle the key structural difficulty of \emph{missing labels} across heads, which arises from the feedback stream since only the explored model receives a label at each step.
Its architecture is illustrated in~\Cref{fig:predictor}.
The overall objective of predictor training is a standard binary cross-entropy loss:
\begin{align}
\label{eq:cross-entropy}
\textstyle
        F(\mathbf{z}) := -\Expectsubbracket{\mathbf{a}_{t}\sim\mathcal{A}}{\frac{1}{M}\sum_{m=1}^M \left( s_{m,t}\log \hat{s}(\mathbf{z}, \mathbf{a}_{t})_m  + (1\!-\! s_{m,t})  \log (1\!-\!\hat{s}(\mathbf{z}, \mathbf{a}_{t})_m) \right) }\! +\! \frac{\mu}{2}\normsq{\mathbf{z}},
\end{align}
where $\mu>0$ is a weight-decay coefficient, $\mathbf{a}_{t}$ is the input of the $t$-th request, and $\hat{s}(\mathbf{z}, \mathbf{a}_{t})_m=\hat{s}_{m,t}$ for predictor parameter $\mathbf{z}$ and request input $\mathbf{a}_{t}$. 

The \textit{uniqueness} of SLARouter's predictor is that it is trained only using one-sided feedback, i.e., feedback from only a single model that is chosen to serve the user. To facilitate this, all the requests with feedback collected from exploration steps are saved in a \textit{predictor training dataset} $\mathcal{D}$. In each exploration step, if feedback is collected, this new sample $(\mathbf{a}_t, s_{\tilde{m},t}, \tilde{m})$ is added to $\mathcal{D}$, where $\tilde{m}$ is the randomly chosen model during exploration; then, a step of stochastic gradient descent (SGD) with a mini-batch $\mathcal{B}$, sampled from $\mathcal{D}$, is performed to update the predictor's parameter~$\mathbf{z}$. The mini-batch size $|\mathcal{B}| := \min\{|\mathcal{D}|; B\}$, where $|\cdot|$ denotes the cardinality of the set and $B$ is the default mini-batch size once $|\mathcal{D}|\geq B$. We denote the mini-batch gradient by $\nabla F(\mathbf{z}, \mathcal{B})$, which is computed by
\begin{align}
    \!\!\!\!\!\nabla\! F(\mathbf{z}, \mathcal{B})\! = - \nabla_{\!\mathbf{z}} \!\!\left[ \textstyle
        \frac{1}{|\mathcal{B}|}\sum_{(\mathbf{a}_t, s_{\tilde{m},t}, \tilde{m})\in\mathcal{B}} \left( s_{\tilde{m},t}\log \hat{s}(\mathbf{z}, \mathbf{a}_{t})_{\tilde{m}} \! + \!(1\!-\! s_{\tilde{m},t})  \log (1\!-\!\hat{s}(\mathbf{z}, \mathbf{a}_{t})_{\tilde{m}}) \right) \right]\! +\! \mu\mathbf{z}.\!\!\!
        \label{eq:mini-batch-gradient}
\end{align}
Due to the one-sided feedback, it is possible that some LLMs do not have samples in a mini-batch. In that case, the predictor parameters directly connected to the outputs corresponding to those models are not updated for that specific mini-batch SGD step, because the gradient with respect to those parameters is zero as there is no loss term corresponding to those models in \eqref{eq:mini-batch-gradient}.

When the feedback availability follows an IID distribution across requests and models, we can easily see that $\nabla F(\mathbf{z}, \mathcal{B})$ is an unbiased estimate of the full gradient $\nabla F(\mathbf{z})$.
Thus, we have addressed the limited feedback problem by \textit{drawing an equivalence between random sampling for SGD, random model selection for exploration, and (random) feedback availability}.

\begin{figure}[!t]
    \centering
    \includegraphics[width=0.7\linewidth]{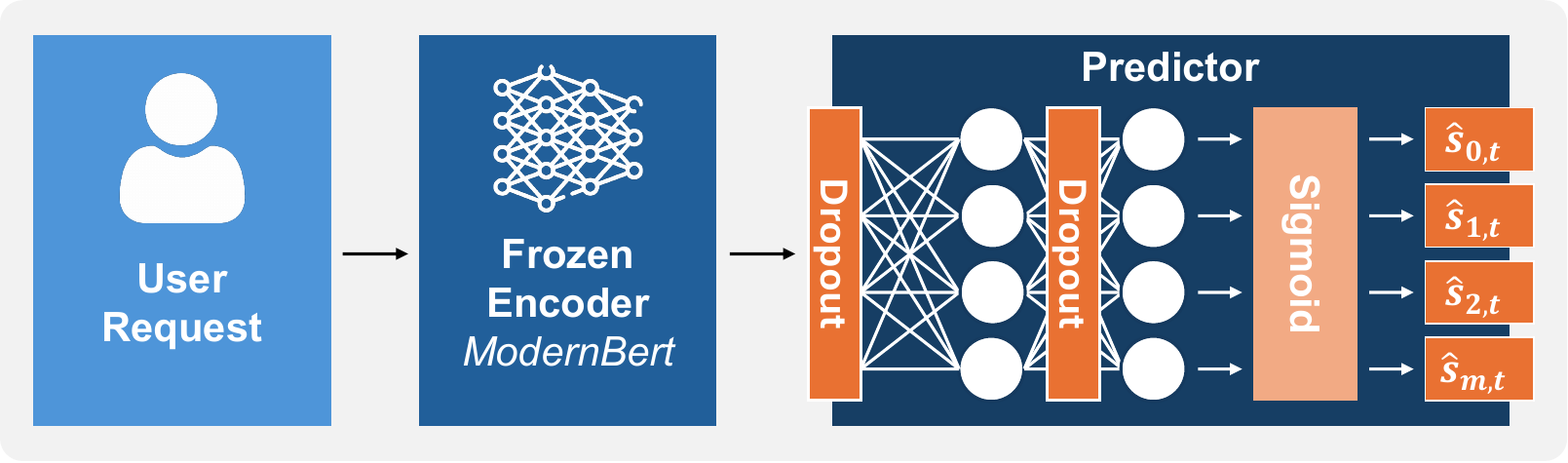}
    \caption{Our predictor uses a frozen ModernBert encoder and a multi-label classification MLP. For training, we apply a masked BCE loss to only train those heads for which we have received user feedback. The predictor output is the per-model probability of satisfying an incoming user request.}
    \label{fig:predictor}
\end{figure}

\textbf{Routing Decision.}
As the predictor improves, routing decisions rely increasingly on its predictions. 
The predictor scores $\hat{s}_{m,t}$ and the virtual queue $Q_t$ jointly determine the \emph{per-request decision problem}:%
\begin{subequations}
\begin{align}
    \textstyle
    \min_{\{y_{m,t}:\forall m\}} \quad V\cdot \textstyle\sum_{m=1}^M & \textstyle 
        y_{m,t} C_{m,t} + Q_t \left(\alpha - \textstyle\sum_{m=1}^M 
        y_{m,t}\hat{s}_{m,t}\right)\mathrm{,} 
        \label{eq:per-request_obj} \\
    \textrm{s.t.}\quad  & \textrm{Constraints \eqref{eq:one_model}, 
        \eqref{eq:binary}}\,.
\end{align}%
\label{eq:per-request-optimization}%
\end{subequations}
The update rule for $Q_t$ is defined in \eqref{eq:queue-update} below.
Parameter $V > 0$ controls the trade-off between cost efficiency and convergence speed. 
A smaller $V$ yields faster SLA compliance at higher initial operating cost, while a larger $V$ prioritizes cost savings at the expense of slower convergence. 

\textbf{Handling Feedback Sparsity.}
Updating $Q_t$ in the standard way requires a feedback signal at every step, but explicit user feedback often arrives at low rates in production systems~\cite{eisenstein2025metrics, feedback_sparsity}.
When $s_{\tilde{m},t}$ is unavailable, we substitute it with the predictor score $\hat{s}_{\tilde{m},t}$ as a proxy.
$Q_t$ evolves as follows:  
\begin{equation}
Q_{t+1} = \begin{cases}
    \max\{0, Q_t + \alpha - s_{\tilde{m}, t} \}, & \textrm{if } s_{\tilde{m}, t}\neq\emptyset\\
    \max\{0, Q_t + \alpha - \hat{s}_{\tilde{m}, t} \}, & \textrm{if } s_{\tilde{m}, t}=\emptyset
\end{cases}.
\label{eq:queue-update}
\end{equation}
This substitution keeps our Lyapunov analysis intact, provided that the predictor estimate is sufficiently accurate. 
This is a condition enforced by the exploration schedule, which guarantees continued predictor improvement and adaptation to the user throughout deployment.

\vspace{-0.5em}
\section{Performance Analysis}
\vspace{-0.5em}

We present the theoretical performance analysis in this section. The full proofs are in Appendix~\ref{appendix:proofs}.

\vspace{-0.5em}
\subsection{Predictor Training}
\vspace{-0.5em}
\begin{assumption}
\label{assumption:predictor_loss}
The predictor training objective $F(\mathbf{z})$ is $L$-smooth and satisfies the Polyak--\L{}ojasiewicz (PL) condition with parameter $\mu > 0$ (and $\mu \leq L$), i.e.,
$
    \tfrac{1}{2}\|\nabla F(\mathbf{z})\|^2 \;\ge\; \mu\bigl(F(\mathbf{z})-F_\mathrm{min}\bigr), \forall\, \mathbf{z},
$
where $F_\mathrm{min} := \min_{\mathbf{z}} F(\mathbf{z})$. Its single-sample stochastic gradient, denoted by $\nabla f(\mathbf{z}, \mathbf{a}, m)$, is unbiased with bounded variance $\sigma^2$, i.e.,
$
    \mathbb{E}\bigl[\nabla f(\mathbf{z}, \mathbf{a}, m) \mid \mathbf{z}\bigr] = \nabla F(\mathbf{z}) \text{ and } 
    \mathbb{E}\bigl[\|\nabla f(\mathbf{z}, \mathbf{a}, m) - \nabla F(\mathbf{z})\|^2 \mid \mathbf{z}\bigr] \leq \sigma^2,
    \forall\, \mathbf{z}, m.
$ The request $\mathbf{a}_t \sim \mathcal{A}$ is IID across~$t$.
\end{assumption}

\Cref{assumption:predictor_loss} is standard in SGD convergence analysis. Let $k$ denote the SGD steps until request $t$. Note that both $k$ and the dataset size $|\mathcal{D}|$ are random variables due to the randomness of exploration. A uniqueness in our setup is that there exist two distinct sources of gradient error:
\begin{align*}
\textstyle
\underbrace{\nabla F(\mathbf{z})}_{\substack{\text{all $M$ models} \\ \text{over full dataset space}}} 
\xrightarrow{|\mathcal{D}| \text{ samples available}} 
\underbrace{\nabla \hat{F}(\mathbf{z})}_{\substack{\text{variance induced by} \\ \text{limited feedback: } \mathcal{O}(\sigma^2/|\mathcal{D}|)}} 
\xrightarrow{B \text{ batch size}} \underbrace{\nabla F(\mathbf{z}, \mathcal{B})}_{\substack{\text{variance induced by} \\ \text{stochastic gradient sampling: } \mathcal{O}(\sigma^2/B)}}.
\end{align*}
The \textit{empirical} ``full'' gradient on the currently collected dataset $\mathcal{D}$, denoted by $\nabla\hat{F}(\mathbf{z}_k)$, is computed over the $|\mathcal{D}|$ available feedback samples, while the update gradient $\nabla F(\mathbf{z}, \mathcal{B})$ further subsamples $B$ samples per SGD step from $\mathcal{D}$. Larger $B$ and $|\mathcal{D}|$ reduce both error sources, as shown in \Cref{lemma:sgd_bound}.

\begin{lemma}
\label{lemma:sgd_bound}Choosing $\eta_k=\min\left\{\frac{4}{\mu (k+1)}, \frac{1}{2L}\right\}$,
after $k\geq 1$ steps of predictor training using SGD with the loss function defined in 
\eqref{eq:cross-entropy}, we have
\begin{align}
    \textstyle\Expectcond{F(\mathbf{z}_k)}{k, \mathcal{D}}-F_\mathrm{min}\leq \mathcal{O}\left(\frac{1}{k^2} + \frac{\sigma^2}{|\mathcal{D}|} + \frac{\sigma^2}{kB}\right).
    \label{eq:proof_sgd_final}
\end{align}
\end{lemma}
The proof involves a unique way of decomposing and bounding the two error sources. \Cref{lemma:sgd_bound} captures the two distinct sources of gradient error that jointly govern the rate at which $F(\mathbf{z}_k)$ approaches its true minimum $F_\mathrm{min}$. 
The first and the last terms are standard for SGD. The second term captures the variance induced by the finite size of dataset $\mathcal{D}$.
The diminishing step size $\eta_k=\Theta(\sfrac{1}{k})$ is the standard theoretical technique for exact convergence to $F_\mathrm{min}$. In practice, we found that a constant rate can already be sufficient for good performance (see \Cref{sec:experiments}).

Next, we focus on the total expectation of $F(\mathbf{z}_k)$, starting with an assumption on feedback availability.

\begin{assumption}
\label{assumption:partial_feedback}
For every request $t$, user feedback is available IID with a fixed probability $\feedbackprob \in (0,1]$.
\end{assumption}

In \Cref{algo:slarouter}, we always have $k = |\mathcal{D}|$, as we only train the predictor when new feedback is available. Our theory can be directly extended to support other schedules though. 
With the exploration probability of $p_t = \min\left(1, \sfrac{c}{\sqrt[4]{t}}\right)$, we then have $\mathbb{E}[k] = \mathbb{E}[\mathcal{D}] = \Theta( t^{3/4}\phi)$. Taking total expectation to the bound in \Cref{lemma:sgd_bound} and noting $k = |\mathcal{D}|$, we then have $\Expectbracket{F(\mathbf{z}_k)}-F_\mathrm{min}\leq \mathcal{O}\left(\sfrac{(1+\sigma^2)}{t^{3/4}\phi}\right).$

Considering characteristics of the cross-entropy loss and applying Markov's inequality, we have the following lemma on satisfaction estimation error that is used as a basis for subsequent results.
\begin{lemma}
    \label{lemma:satisfaction_error_bound_with_delta}
    Under the same condition as for Lemma~\ref{lemma:sgd_bound}, there exists some positive $\beta\leq M$, so that
    \begin{align}
        \textstyle\Expectbracket{\max_m|\hat{s}_{m,t} - s_{m,t}|} &\textstyle\leq \beta F_\mathrm{min} 
        + \mathcal{O}\left(\frac{\beta (1+\sigma^2)}{ t^{\sfrac{3}{4}} \feedbackprob}\right);\\
        \textstyle\textstyle\max_m|\hat{s}_{m,t} - s_{m,t}| &\textstyle\leq \beta F_\mathrm{min} 
        + \mathcal{O}\left(\frac{\beta (1+\sigma^2)}{ t^{\sfrac{3}{4}} \feedbackprob \delta}\right), \text{with probability $1-\delta$}.
    \end{align}
\end{lemma}
This shows the estimation error of user satisfaction is upper-bounded at high probability with an error floor of $\beta F_{\mathrm{min}}$, where the last term vanishes as $t$ gets large. 
The theory uses $\beta=M$ in the proof, while in practice $\beta$ can be much smaller than $M$.

\vspace{-0.5em}
\subsection{SLA Satisfaction}
\vspace{-0.5em}

\begin{assumption}
\label{assumption:queue_length}
We assume that $\beta F_\mathrm{min} < \frac{1}{3}$, and there exists at least one model $m$ such that $\Pr\{s_{m,t} = 1\} \geq q > \alpha$ and $(1-\feedbackprob)\beta F_\mathrm{min} < q-\alpha$, for some $q\in(\alpha, 1]$.
\end{assumption}
This assumption states that the predictor architecture is capable enough so that an optimally trained predictor has a small loss. It also states that there is at least one model for which the satisfaction probability is greater than $\alpha$, since otherwise the problem becomes infeasible. 
The last condition weakens as feedback becomes denser and vanishes entirely in the full-feedback case $\feedbackprob=1$. We also assume that the conditions in Lemmas~\ref{lemma:sgd_bound} and \ref{lemma:satisfaction_error_bound_with_delta} hold in the following.

\begin{theorem}
\label{theorem:max_queue_length}
    Let $\rho := \frac{q-\alpha-(1-\feedbackprob)\beta F_\mathrm{min}}{2} > 0$, $\psi := \Theta\left(\left(\frac{\beta(1+\sigma^2)}{\feedbackprob\, \rho \left(\frac{1}{3} - \beta F_\mathrm{min}\right)}\right)^{\frac{4}{3}}\right)$, $\gamma := \max\left\{\alpha\psi;\, 3V\Delta_C\right\}$.
    For any $t \geq 1$, we have these upper bounds on the virtual queue length:
    \begin{align}
        \textstyle\Expectbracket{Q_t} \leq \gamma + \mathcal{O}\left(\sqrt{t}\right)
        \quad\text{and}\quad
        \frac{1}{t}\sum_{\tau=1}^{t}\Expectbracket{Q_\tau} \leq \mathcal{O}\left(\gamma + \sfrac{1}{\rho}\right),
    \end{align}
    where $\Delta_C := C_{\max} - C_{\min}$, and $C_{\max}$ and $C_{\min}$ are the maximum and minimum costs, respectively.
\end{theorem}

The proof of \Cref{theorem:max_queue_length} is based on a unique way of analyzing the terms involving feedback availability and the two variants of queue updates in \eqref{eq:queue-update}.
The feedback probability $\phi$ appears in the denominator in the definition of $\psi$ (and hence $\gamma$), which verifies that the queue tends to be longer when the feedback rate is low.
Based on \eqref{eq:queue-update} and \Cref{lemma:satisfaction_error_bound_with_delta}, we obtain the following corollary.

\begin{corollary}
\label{corollary:constraint_satisfaction_at_t}
    We have the following upper bound on the time-averaged constraint violation:
    \begin{align}
        \!\!\!\textstyle\alpha - \frac{1}{T}\sum_{t=1}^T \sum_{m=1}^M \Expectbracket{y_{m,t} s_{m,t}}
        &\textstyle\leq \frac{\Expectbracket{Q_{T+1}}}{T} + (1\!-\!\phi)\beta F_{\mathrm{min}} + \mathcal{O}\!\left(\frac{(1\!-\!\phi)\beta(1\!+\!\sigma^2)}{\phi\, T^{\sfrac{3}{4}}}\right)\nonumber\\
        &\textstyle\leq \mathcal{O}\!\left(\frac{\gamma}{T} + \frac{1}{\sqrt{T}}\right) + (1\!-\!\phi)\beta F_{\mathrm{min}}.\!\!\!
    \end{align}
\end{corollary}

As $T \to \infty$, the constraint violation approaches the error floor $(1-\phi)\beta F_{\mathrm{min}}$. When the predictor is more accurate or the feedback rate is high, the constraint violation becomes less severe. For full feedback with $\phi=1$, we recover the result in \cite{woisetschlager2025messplus}. If full SLA compliance at level $\alpha'$ is required after a finite number of requests $T_0$, one may set $\alpha$ slightly above $\alpha'$ such that the upper bound is less than $\alpha - \alpha'$ after $T_0$ steps.

\vspace{-0.5em}
\subsection{Cost Optimality}
\vspace{-0.5em}

\begin{theorem}
    \label{theorem:cost_optimality_bound}
    For $\{y_{m,t} : \forall m, t\}$ obtained from the \algname{} algorithm (\Cref{algo:slarouter}), there exists a learning rate schedule $\{\eta_k : \forall k\}$ such that
    \begin{align}
        \textstyle \Expectbracket{\frac{1}{T}\sum_{t=1}^T \sum_{m=1}^M y_{m,t} C_{m,t}}
        \leq C^{\mathrm{OPT}} + \mathcal{O}\!\left(\frac{\beta}{\sqrt[4]{T}} + \beta F_{\mathrm{min}} + \frac{1}{V}\right),
    \end{align}
    where $C^{\mathrm{OPT}}$ is the optimal solution to \eqref{eq:original-problem} obtained from an optimal stationary policy that has full statistical knowledge of requests $t \in \{1, \ldots, T\}$.
\end{theorem}

The proof bounds the drift-plus-penalty expression and incorporates \Cref{lemma:satisfaction_error_bound_with_delta} and \Cref{theorem:max_queue_length}. The optimality gap gets smaller as $T$ is larger, up to the irreducible predictor error $\beta F_{\mathrm{min}}$ and the $\sfrac{1}{V}$ term. 
The parameter $V$ governs the trade-off between cost optimality and SLA satisfaction, where a larger $V$ tightens the optimality gap but slows constraint satisfaction, and vice versa. The residual term $\beta F_{\mathrm{min}}$ reflects the irreducible predictor error. When the predictor is expressive, $F_{\mathrm{min}}$ is small.

\vspace{-0.8em}
\section{Practical Considerations}
\label{sec:practical_considerations}
\vspace{-0.8em}

\textbf{Positive-Class Reweighting.}
In practice, the distribution of user feedback is rarely balanced.
Users tend to provide predominantly negative feedback to indicate whenever the LLM does not serve their requests sufficiently well.
Without accounting for this potential imbalance, the predictor might converge to a trivial constant predictor in the worst case.
To counteract this, we use per-model positive-class weights to upweight sparse positive feedback.
For each model $m$, we maintain running counts of positive and negative observations across all collected feedback and compute $p_m = \sfrac{N_m^-}{N_m^+},$
where $N_m^+$ and $N_m^-$ denote the number of positive and negative feedback samples for model $m$ in $\mathcal{D}$.
At each SGD step, $p_m$ is applied as the positive-class weight for the corresponding output in the masked BCE loss over the mini-batch B.
This rescales the contribution of the minority class proportionally, ensuring the predictor remains discriminative regardless of label skew.\footnote{We note that in our benchmark evaluation we do not observe notable imbalance. However, in production deployments imbalance can be significant, making this correction an important safeguard.}

\textbf{Dynamic Adaptation of $V$.}
The parameter $V$ controls the trade-off between cost efficiency and convergence speed towards the satisfaction target $\alpha$. 
A small $V$ prioritizes rapid SLA compliance by routing to expensive models early, and vice versa.
From our theoretical analysis, we see that $V$ depends on the cost spread between largest and smallest model in a zoo, making manual tuning of $V$ fragile.
Instead, we adapt $V$ with a data-driven mechanism $V = (Q_{\max} \cdot \varepsilon) / \Delta C$, where $Q_{\max}$ is the maximum tolerated queue length, $\varepsilon > 0$ is a cost-sensitivity parameter, and $\Delta C$ is the observed mean cost difference between the cheapest and most expensive model.
Intuitively, a wider cost spread yields a smaller $V$, lowering the pressure toward cheaper models and allowing for timely SLA convergence. 
This removes a hard-to-tune hyperparameter and improves the adaptivity of \algname.
We demonstrate the practical benefit in \Cref{sec:experimental_results}.

\begin{table*}[tbp]
    \centering
    \caption{
    Main results. 
    \textbf{\algname} is the cheapest routing technique that guarantees a minimum user satisfaction rate. 
    \textcolor{darkgreen}{Green} indicates the router or model satisfies the SLA ($\alpha$) and the cheapest one is \underline{underlined}. CARROT and Causal Router results include offline classifier training costs. 
    }
    \label{tab:main_results}    
    \newcommand{\mcc}[1]{\makecell[c]{#1}}
    \newcommand{\mcl}[1]{\makecell[l]{#1}}
    \setlength{\tabcolsep}{2pt}
\renewcommand{\arraystretch}{0.82}
\resizebox{\textwidth}{!}{%
\begin{tabular}{l *{3}{ccc}}
\toprule
& \multicolumn{3}{c}{\textbf{ACPBench} ($\alpha$=65\%)} & \multicolumn{3}{c}{\textbf{ARC Challenge} ($\alpha$=52\%)} & \multicolumn{3}{c}{\textbf{ARC Easy} ($\alpha$=80\%)} \\
\cmidrule(lr){2-4}\cmidrule(lr){5-7}\cmidrule(lr){8-10}
\textbf{Method} & \thead{Operating\\Cost (MJ)} & \thead{Request.\\Satisf. (\%)} & \thead{Model Call Ratio\\(2B/9B/35B/122B)} & \thead{Operating\\Cost (MJ)} & \thead{Request.\\Satisf. (\%)} & \thead{Model Call Ratio\\(2B/9B/35B/122B)} & \thead{Operating\\Cost (MJ)} & \thead{Request.\\Satisf. (\%)} & \thead{Model Call Ratio\\(2B/9B/35B/122B)} \\
\cmidrule(lr){1-1}\cmidrule(lr){2-4}\cmidrule(lr){5-7}\cmidrule(lr){8-10}
2B only & 1.01 & \textcolor{darkred}{44.8} & 100\%/0\%/0\%/0\% & 0.06 & \textcolor{darkred}{38.2} & 100\%/0\%/0\%/0\% & 0.11 & \textcolor{darkred}{70.9} & 100\%/0\%/0\%/0\% \\
9B only & 1.54 & \textcolor{darkred}{59.2} & 0\%/100\%/0\%/0\% & 0.07 & \underline{\textcolor{darkgreen}{53.2}} & 0\%/100\%/0\%/0\% & 0.14 & \underline{\textcolor{darkgreen}{81.2}} & 0\%/100\%/0\%/0\% \\
35B only & 2.94 & \underline{\textcolor{darkgreen}{66.2}} & 0\%/0\%/100\%/0\% & 0.07 & \textcolor{darkgreen}{58.9} & 0\%/0\%/100\%/0\% & 0.15 & \textcolor{darkgreen}{84.1} & 0\%/0\%/100\%/0\% \\
122B only & 7.19 & \textcolor{darkgreen}{73.2} & 0\%/0\%/0\%/100\% & 0.22 & \textcolor{darkgreen}{61.0} & 0\%/0\%/0\%/100\% & 0.41 & \textcolor{darkgreen}{84.9} & 0\%/0\%/0\%/100\% \\
\cmidrule(lr){1-1}\cmidrule(lr){2-4}\cmidrule(lr){5-7}\cmidrule(lr){8-10}
CARROT & 2.87$_{\pm0.65}$ & \textcolor{darkgreen}{66.5$_{\pm2.4}$} & 0\%/22\%/60\%/18\% & 0.16$_{\pm0.01}$ & \textcolor{darkgreen}{55.2$_{\pm1.5}$} & 12\%/31\%/57\%/0\% & 0.32$_{\pm0.01}$ & \textcolor{darkgreen}{80.6$_{\pm1.9}$} & 25\%/17\%/58\%/0\% \\
Causal & 6.84$_{\pm0.50}$ & \textcolor{darkgreen}{71.8$_{\pm1.2}$} & 0\%/0\%/14\%/86\% & 0.15$_{\pm0.01}$ & \textcolor{darkgreen}{52.5$_{\pm1.1}$} & 6\%/91\%/3\%/0\% & 0.31$_{\pm0.01}$ & \textcolor{darkgreen}{80.5$_{\pm1.1}$} & 5\%/92\%/3\%/0\% \\
MESS+ & 3.20$_{\pm0.98}$ & \textcolor{darkred}{61.2$_{\pm3.9}$} & 19\%/34\%/19\%/27\% & 0.12$_{\pm0.02}$ & \textcolor{darkred}{51.8$_{\pm1.6}$} & 31\%/25\%/18\%/26\% & 0.16$_{\pm0.01}$ & \textcolor{darkred}{78.2$_{\pm1.1}$} & 36\%/33\%/28\%/3\% \\
\textbf{SLARouter (ours)} & 2.71$_{\pm0.38}$ & \textbf{\underline{\textcolor{darkgreen}{65.1$_{\pm2.6}$}}} & 14\%/10\%/62\%/14\% & 0.10$_{\pm0.04}$ & \textbf{\underline{\textcolor{darkgreen}{53.0$_{\pm3.2}$}}} & 26\%/28\%/26\%/20\% & 0.21$_{\pm0.05}$ & \textbf{\underline{\textcolor{darkgreen}{81.6$_{\pm0.3}$}}} & 20\%/14\%/42\%/24\% \\
\bottomrule
\end{tabular}%
}

\resizebox{\textwidth}{!}{%
\begin{tabular}{l *{3}{ccc}}
\toprule
& \multicolumn{3}{c}{\textbf{BoolQ} ($\alpha$=85\%)} & \multicolumn{3}{c}{\textbf{GPQA} ($\alpha$=45\%)} & \multicolumn{3}{c}{\textbf{GSM8K} ($\alpha$=83\%)} \\
\cmidrule(lr){2-4}\cmidrule(lr){5-7}\cmidrule(lr){8-10}
\textbf{Method} & \thead{Operating\\Cost (MJ)} & \thead{Request.\\Satisf. (\%)} & \thead{Model Call Ratio\\(2B/9B/35B/122B)} & \thead{Operating\\Cost (MJ)} & \thead{Request.\\Satisf. (\%)} & \thead{Model Call Ratio\\(2B/9B/35B/122B)} & \thead{Operating\\Cost (MJ)} & \thead{Request.\\Satisf. (\%)} & \thead{Model Call Ratio\\(2B/9B/35B/122B)} \\
\cmidrule(lr){1-1}\cmidrule(lr){2-4}\cmidrule(lr){5-7}\cmidrule(lr){8-10}
2B only & 0.09 & \textcolor{darkred}{72.0} & 100\%/0\%/0\%/0\% & 0.02 & \textcolor{darkred}{32.5} & 100\%/0\%/0\%/0\% & 0.56 & \textcolor{darkred}{56.9} & 100\%/0\%/0\%/0\% \\
9B only & 0.11 & \underline{\textcolor{darkgreen}{88.5}} & 0\%/100\%/0\%/0\% & 0.04 & \textcolor{darkred}{38.8} & 0\%/100\%/0\%/0\% & 0.92 & \underline{\textcolor{darkgreen}{86.4}} & 0\%/100\%/0\%/0\% \\
35B only & 0.13 & \textcolor{darkgreen}{88.9} & 0\%/0\%/100\%/0\% & 0.04 & \textcolor{darkred}{44.3} & 0\%/0\%/100\%/0\% & 1.74 & \textcolor{darkgreen}{87.9} & 0\%/0\%/100\%/0\% \\
122B only & 0.36 & \textcolor{darkgreen}{87.1} & 0\%/0\%/0\%/100\% & 0.11 & \underline{\textcolor{darkgreen}{49.7}} & 0\%/0\%/0\%/100\% & 5.59 & \textcolor{darkgreen}{84.2} & 0\%/0\%/0\%/100\% \\
\cmidrule(lr){1-1}\cmidrule(lr){2-4}\cmidrule(lr){5-7}\cmidrule(lr){8-10}
CARROT & 0.38$_{\pm0.01}$ & \textcolor{darkgreen}{88.0$_{\pm1.2}$} & 5\%/28\%/67\%/0\% & 0.09$_{\pm0.04}$ & \textcolor{darkgreen}{46.1$_{\pm3.2}$} & 0\%/0\%/66\%/34\% & 2.04$_{\pm0.05}$ & \textcolor{darkgreen}{86.6$_{\pm1.7}$} & 2\%/56\%/42\%/0\% \\
Causal & 0.36$_{\pm0.01}$ & \textcolor{darkgreen}{88.7$_{\pm0.3}$} & 0\%/90\%/10\%/0\% & 0.11$_{\pm0.04}$ & \textcolor{darkgreen}{47.9$_{\pm3.2}$} & 0\%/0\%/33\%/67\% & 4.16$_{\pm2.00}$ & \textcolor{darkgreen}{85.9$_{\pm1.9}$} & 0\%/0\%/39\%/61\% \\
MESS+ & 0.21$_{\pm0.04}$ & \textcolor{darkred}{83.4$_{\pm1.9}$} & 31\%/22\%/13\%/35\% & 0.07$_{\pm0.04}$ & \textcolor{darkgreen}{45.8$_{\pm2.9}$} & 4\%/0\%/62\%/34\% & 1.40$_{\pm0.32}$ & \textcolor{darkred}{74.0$_{\pm8.0}$} & 42\%/43\%/4\%/11\% \\
\textbf{SLARouter (ours)} & 0.32$_{\pm0.07}$ & \textbf{\underline{\textcolor{darkgreen}{85.6$_{\pm3.0}$}}} & 24\%/12\%/26\%/38\% & 0.06$_{\pm0.04}$ & \textbf{\underline{\textcolor{darkgreen}{45.0$_{\pm1.6}$}}} & 7\%/1\%/63\%/29\% & 1.65$_{\pm0.46}$ & \textbf{\underline{\textcolor{darkgreen}{83.1$_{\pm1.9}$}}} & 18\%/24\%/49\%/9\% \\
\bottomrule
\end{tabular}%
}

\resizebox{\textwidth}{!}{%
\begin{tabular}{l *{3}{ccc}}
\toprule
& \multicolumn{3}{c}{\textbf{LAMBADA} ($\alpha$=65\%)} & \multicolumn{3}{c}{\textbf{MMLU} ($\alpha$=75\%)} & \multicolumn{3}{c}{\textbf{SciQ} ($\alpha$=96\%)} \\
\cmidrule(lr){2-4}\cmidrule(lr){5-7}\cmidrule(lr){8-10}
\textbf{Method} & \thead{Operating\\Cost (MJ)} & \thead{Request.\\Satisf. (\%)} & \thead{Model Call Ratio\\(2B/9B/35B/122B)} & \thead{Operating\\Cost (MJ)} & \thead{Request.\\Satisf. (\%)} & \thead{Model Call Ratio\\(2B/9B/35B/122B)} & \thead{Operating\\Cost (MJ)} & \thead{Request.\\Satisf. (\%)} & \thead{Model Call Ratio\\(2B/9B/35B/122B)} \\
\cmidrule(lr){1-1}\cmidrule(lr){2-4}\cmidrule(lr){5-7}\cmidrule(lr){8-10}
2B only & 0.06 & \textcolor{darkred}{48.0} & 100\%/0\%/0\%/0\% & 0.41 & \textcolor{darkred}{59.7} & 100\%/0\%/0\%/0\% & 0.05 & \textcolor{darkred}{94.0} & 100\%/0\%/0\%/0\% \\
9B only & 0.08 & \textcolor{darkred}{63.2} & 0\%/100\%/0\%/0\% & 0.55 & \underline{\textcolor{darkgreen}{75.9}} & 0\%/100\%/0\%/0\% & 0.07 & \underline{\textcolor{darkgreen}{97.0}} & 0\%/100\%/0\%/0\% \\
35B only & 0.09 & \underline{\textcolor{darkgreen}{69.2}} & 0\%/0\%/100\%/0\% & 0.65 & \textcolor{darkgreen}{82.7} & 0\%/0\%/100\%/0\% & 0.08 & \textcolor{darkgreen}{96.8} & 0\%/0\%/100\%/0\% \\
122B only & 0.27 & \textcolor{darkgreen}{69.2} & 0\%/0\%/0\%/100\% & 1.98 & \textcolor{darkgreen}{87.6} & 0\%/0\%/0\%/100\% & 0.22 & \textcolor{darkgreen}{97.3} & 0\%/0\%/0\%/100\% \\
\cmidrule(lr){1-1}\cmidrule(lr){2-4}\cmidrule(lr){5-7}\cmidrule(lr){8-10}
CARROT & 0.50$_{\pm0.01}$ & \textcolor{darkgreen}{67.8$_{\pm0.6}$} & 1\%/20\%/78\%/1\% & 1.31$_{\pm0.01}$ & \textcolor{darkgreen}{78.6$_{\pm1.0}$} & 4\%/47\%/49\%/0\% & 0.14$_{\pm0.01}$ & \textcolor{darkgreen}{96.6$_{\pm0.1}$} & 30\%/35\%/35\%/0\% \\
Causal & 0.47$_{\pm0.01}$ & \textcolor{darkgreen}{65.2$_{\pm3.5}$} & 0\%/67\%/33\%/0\% & 1.42$_{\pm0.03}$ & \textcolor{darkgreen}{76.3$_{\pm0.9}$} & 0\%/12\%/88\%/0\% & 0.14$_{\pm0.01}$ & \textcolor{darkgreen}{96.9$_{\pm0.1}$} & 2\%/92\%/6\%/0\% \\
MESS+ & 0.13$_{\pm0.03}$ & \textbf{\underline{\textcolor{darkgreen}{66.7$_{\pm0.8}$}}} & 11\%/5\%/69\%/15\% & 0.96$_{\pm0.22}$ & \textcolor{darkred}{73.9$_{\pm2.3}$} & 35\%/29\%/10\%/26\% & 0.10$_{\pm0.01}$ & \textcolor{darkgreen}{96.5$_{\pm0.1}$} & 24\%/39\%/32\%/4\% \\
\textbf{SLARouter (ours)} & 0.17$_{\pm0.04}$ & \textcolor{darkgreen}{65.6$_{\pm1.3}$} & 11\%/32\%/21\%/36\% & 0.97$_{\pm0.06}$ & \textbf{\underline{\textcolor{darkgreen}{77.7$_{\pm0.3}$}}} & 21\%/23\%/30\%/26\% & 0.09$_{\pm0.03}$ & \textbf{\underline{\textcolor{darkgreen}{96.3$_{\pm0.5}$}}} & 29\%/29\%/28\%/13\% \\
\bottomrule
\end{tabular}%
}

\resizebox{\textwidth}{!}{%
\begin{tabular}{l *{3}{ccc}}
\toprule
& \multicolumn{3}{c}{\textbf{SocialIQa} ($\alpha$=48\%)} & \multicolumn{3}{c}{\textbf{WinoGrande} ($\alpha$=73\%)} & \multicolumn{3}{c}{\textbf{Avg. across 11 benchmarks} ($\alpha$=70\%)} \\
\cmidrule(lr){2-4}\cmidrule(lr){5-7}\cmidrule(lr){8-10}
\textbf{Method} & \thead{Operating\\Cost (MJ)} & \thead{Request.\\Satisf. (\%)} & \thead{Model Call Ratio\\(2B/9B/35B/122B)} & \thead{Operating\\Cost (MJ)} & \thead{Request.\\Satisf. (\%)} & \thead{Model Call Ratio\\(2B/9B/35B/122B)} & \thead{Operating\\Cost (MJ)} & \thead{Request.\\Satisf. (\%)} & \thead{Model Call Ratio\\(2B/9B/35B/122B)} \\
\cmidrule(lr){1-1}\cmidrule(lr){2-4}\cmidrule(lr){5-7}\cmidrule(lr){8-10}
2B only & 0.06 & \textcolor{darkred}{41.0} & 100\%/0\%/0\%/0\% & 0.03 & \textcolor{darkred}{63.2} & 100\%/0\%/0\%/0\% & 0.22 & \textcolor{darkred}{56.5} & 100\%/0\%/0\%/0\% \\
9B only & 0.08 & \underline{\textcolor{darkgreen}{48.1}} & 0\%/100\%/0\%/0\% & 0.04 & \textcolor{darkred}{72.1} & 0\%/100\%/0\%/0\% & 0.33 & \textcolor{darkred}{69.4} & 0\%/100\%/0\%/0\% \\
35B only & 0.09 & \textcolor{darkgreen}{51.9} & 0\%/0\%/100\%/0\% & 0.04 & \underline{\textcolor{darkgreen}{74.6}} & 0\%/0\%/100\%/0\% & 0.55 & \underline{\textcolor{darkgreen}{73.2}} & 0\%/0\%/100\%/0\% \\
122B only & 0.27 & \textcolor{darkgreen}{53.6} & 0\%/0\%/0\%/100\% & 0.12 & \textcolor{darkgreen}{76.3} & 0\%/0\%/0\%/100\% & 1.52 & \textcolor{darkgreen}{74.9} & 0\%/0\%/0\%/100\% \\
\cmidrule(lr){1-1}\cmidrule(lr){2-4}\cmidrule(lr){5-7}\cmidrule(lr){8-10}
CARROT & 0.22$_{\pm0.01}$ & \textcolor{darkgreen}{48.6$_{\pm1.3}$} & 25\%/21\%/54\%/0\% & 0.14$_{\pm0.01}$ & \textcolor{darkgreen}{73.5$_{\pm0.1}$} & 3\%/18\%/79\%/0\% & 0.74$_{\pm0.07}$ & \textcolor{darkgreen}{71.6$_{\pm1.4}$} & 10\%/30\%/58\%/2\% \\
Causal & 0.22$_{\pm0.01}$ & \textcolor{darkgreen}{48.6$_{\pm2.9}$} & 13\%/54\%/33\%/0\% & 0.19$_{\pm0.05}$ & \textcolor{darkgreen}{75.7$_{\pm1.0}$} & 0\%/0\%/33\%/67\% & 1.31$_{\pm0.25}$ & \textcolor{darkgreen}{71.8$_{\pm1.6}$} & 2\%/45\%/27\%/26\% \\
MESS+ & 0.14$_{\pm0.03}$ & \underline{\textbf{\textcolor{darkgreen}{49.2$_{\pm1.8}$}}} & 27\%/18\%/30\%/25\% & 0.07$_{\pm0.03}$ & \textbf{\underline{\textcolor{darkgreen}{74.9$_{\pm0.7}$}}} & 3\%/1\%/70\%/26\% & 0.59$_{\pm0.16}$ & \textcolor{darkred}{68.7$_{\pm2.3}$} & 24\%/23\%/32\%/21\% \\
\textbf{SLARouter (ours)} & 0.15$_{\pm0.05}$ & \textcolor{darkgreen}{49.5$_{\pm1.1}$} & 23\%/10\%/34\%/33\% & 0.07$_{\pm0.02}$ & \textcolor{darkgreen}{74.7$_{\pm0.3}$} & 8\%/2\%/50\%/39\% & 0.59$_{\pm0.11}$ & \underline{\textbf{\textcolor{darkgreen}{70.66$_{\pm1.5}$}}} & 18\%/17\%/39\%/26\% \\
\bottomrule
\end{tabular}%
}

\end{table*}

\section{Experiments}
\label{sec:experiments}
We underpin our theoretical results by an extensive empirical evaluation across a numerous LLM benchmarks.
Our code is available on GitHub\footnote{
Please reach out for repo access while the preprint is under review.
}. 
Details on reproducibility are in the appendix.

\subsection{Design}

\textbf{Model Zoo \& Benchmarks.}
We construct our model zoo from the Qwen 3.5 family~\cite{qwen3}: two dense models (2B, 9B) and two mixture-of-experts models (35B, 122B).
Benchmarks span commonsense reasoning (MMLU, SocialIQa, WinoGrande)~\cite{hendryckstest2021mmlu,sap2019social,sakaguchi2020winogrande}, scientific knowledge (ARC Challenge, ARC Easy, SciQ, GPQA)~\cite{clark2018arc,welbl2017sciq,gpqa}, mathematical reasoning (GSM8K)~\cite{cobbe2021gsm8k}, reading comprehension (BoolQ, LAMBADA)~\cite{clark2019boolq,paperno-etal-2016-lambada}, and agentic tasks (ACPBench)~\cite{acpbench}.

\textbf{Predictor.}
We use SGD, a learning rate of $0.006$, a weight decay of $0.01$, a minibatch size of $B=16$, and set the exploration probability parameter $c=0.1$~\cite{woisetschlager2025messplus}, if not specified otherwise.

\textbf{Service Level Agreement.}
We define $\alpha$ per benchmark within the capabilities of the model zoo and set $V = 0.00001$ to ensure timely SLA compliance.
We explore data-driven $V$ (Section~\ref{sec:practical_considerations}) with $Q_{\mathrm{max}} = 30$ and $\varepsilon = 0.001$, which is desirable for settings with high user satisfaction requirements.

\textbf{Objective.}
We minimize the per-inference energy consumption (MJ) under constraint $\alpha$, and report the model call ratio per benchmark to characterize routing behavior.

\textbf{Observational Data.}
We set the user feedback rate to $0.2$ for all experiments. 
This is primarily to allow for sufficient training data for the CARROT and Causal Router classifiers. 
Below this rate a fair comparison is not possible.
However, we do conduct an ablation study on \algname to show its performance under lower feedback rates than $0.2$.

\textbf{Baselines.}
We compare \algname against CARROT~\cite{somerstep2025carrot}, Causal LLM Routing~\cite{tsiourvas2025causal}, and MESS+~\cite{woisetschlager2025messplus}.
In our comparison, we include the offline classifier training costs in the total operating cost.
For MESS+, we adopt the same $V$ and $c$ settings as what we use for \algname.

\begin{table}[t]
\centering
    \caption{%
        Introducing data-driven selection of $V$ improves operational efficiency while still complying with our SLA $\alpha$. It also remove a hyperparameter and makes \algname more practical to use. 
    }
    \label{tab:slarouter_fixed_vs_driven}
    \setlength{\tabcolsep}{5pt}
\renewcommand{\arraystretch}{0.95}
\resizebox{0.9\textwidth}{!}{%
\begin{tabular}{l ccc ccc}
\toprule
& \multicolumn{3}{c}{\textbf{ACPBench} ($\alpha$=65\%)} & \multicolumn{3}{c}{\textbf{GPQA} ($\alpha$=45\%)} \\
\cmidrule(lr){2-4}\cmidrule(lr){5-7}
\textbf{SLARouter variant} & \thead{Operating\\Cost (MJ)} & \thead{Post-Conv.\\Accuracy} & \thead{Model Call Ratio\\(2B/9B/35B/122B)} & \thead{Operating\\Cost (MJ)} & \thead{Post-Conv.\\Accuracy} & \thead{Model Call Ratio\\(2B/9B/35B/122B)} \\
\cmidrule(lr){1-1}\cmidrule(lr){2-4}\cmidrule(lr){5-7}
Best fixed $V$ & $2.712_{\pm0.12}$ & $65.1_{\pm0.2}$ & 10\%/20\%/58\%/12\% & $0.063_{\pm0.01}$ & $45.1_{\pm0.2}$ & 5\%/2\%/56\%/37\% \\
Data-driven & $2.641_{\pm0.20}$ & $65.1_{\pm0.3}$ & 10\%/20\%/58\%/12\% & $0.060_{\pm0.01}$ & $45.1_{\pm0.1}$ & 4\%/3\%/56\%/37\% \\
\midrule
$\Delta$ (rel.) & \textbf{$-3\%$} & $0.0\%$ & --- & \textbf{$-5\%$} & $0.0\%$ & --- \\
\bottomrule
\end{tabular}%
}

\end{table}

\vspace{-1em}
\subsection{Results}
\label{sec:experimental_results}
\vspace{-0.5em}

\textbf{SLA Guarantees ($\alpha$).}
\emph{\algname is the only method that satisfies $\alpha$ without careful per-benchmark hyperparameter tuning}, by maintaining accurate constraint tracking through predictor-estimated queue updates under sparse feedback.
MESS+ satisfies $\alpha$ on 5 out of 11 benchmarks, with failures attributable to feedback sparsity stalling $Q$ updates under its assumption of complete, balanced signals~(\Cref{tab:main_results}).
CARROT and the Causal router are competitive and also satisfy $\alpha$, but require careful per-benchmark tuning of $\mu$ and $\lambda$ to do so. 
In our experiments, we observe that CARROT and the Causal router both require hyperparameter configurations that strongly prioritize user satisfaction over cost to be able to meet the $\alpha$ requirement. Further details are available in the appendix.

\textbf{Cost Efficiency.}
\emph{\algname achieves the best cost-quality trade-off.} 
It simultaneously satisfies $\alpha$ and minimizes operating cost.
Among methods that meet $\alpha$, \algname is $1.25\times$ and $2.22\times$ cheaper than CARROT and the Causal router, respectively.
This is due to the interplay between the predictor and the adaptation of $Q$ at runtime. 
MESS+ matches \algname on cost but undershoots $\alpha$ on several benchmarks, confirming that low operating cost alone is insufficient without reliable constraint tracking.
The cost advantage is most pronounced on demanding workloads.
On ACPBench, \algname reduces cost by $1.06\times$ relative to CARROT and by over $2.5\times$ compared to the Causal router.
On easier benchmarks (e.g., ARC Easy), differences are smaller as mid-range models suffice.

\textbf{Satisfaction Accuracy and Margin.}
\emph{\algname achieves the highest post-convergence satisfaction rate among adaptive methods on 7 of 11 benchmarks, with a tighter margin above $\alpha$ than competitors}, confirming that its gains stem from more effective use of available signals rather than conservative routing.
Unlike MESS+, whose queue underestimates satisfaction shortfalls under sparse feedback, \algname's virtual queue substitutes predictor estimates to maintain uninterrupted constraint tracking throughout deployment.
CARROT and the Causal router are competitive, but their margins are sensitive to the benchmark-specific choice of their cost sensitivity parameter.

\textbf{Data-Driven Adaptation of Cost-Sensitivity $V$.}
\emph{Using data-driven $V$ eliminates the need for grid search while achieving an overall better SLA convergence speed under relevant cost savings.}
On challenging tasks like ACPBench and GPQA, the data-driven approach reduces operating cost by 3\% and 5\%, respectively, while achieving identical $\alpha$ compliance~(\Cref{tab:slarouter_fixed_vs_driven}).
A detailed comparison against a fine-grained grid search is available in the appendix.

\textbf{Predictor Quality under Sparse Feedback.}
\emph{Synthetic feedback collected from an LLM-based judge is a well-suited supplement for missing user feedback.}
When users do not provide feedback, \algname falls back to synthetic signals from GPT-5.4 with reasoning enabled, which closely mirrors real user feedback with 94.4\% and 87.9\% agreement on ACPBench and GPQA, respectively.
Our analysis shows that the predictor requires relatively few feedback samples to be able to sufficiently estimate the model fit for an incoming request. 
Specifically, we observe that SLA violations drop significantly once feedback from as few as 90 and 35 interactions is available for ACPBench and GPQA, respectively. 
Given the sample size of each benchmark, this is equal to a 5\% and 8\% feedback rate, respectively. 
The judge comes with no additional burden on end users~(\Cref{fig:judge_experiments}).
Details on the judge and alignment with ground truth data is in the appendix.

\begin{figure}[!t]
    \centering
    \begin{subfigure}{0.49\textwidth}
        \centering
        \includegraphics[width=\linewidth]{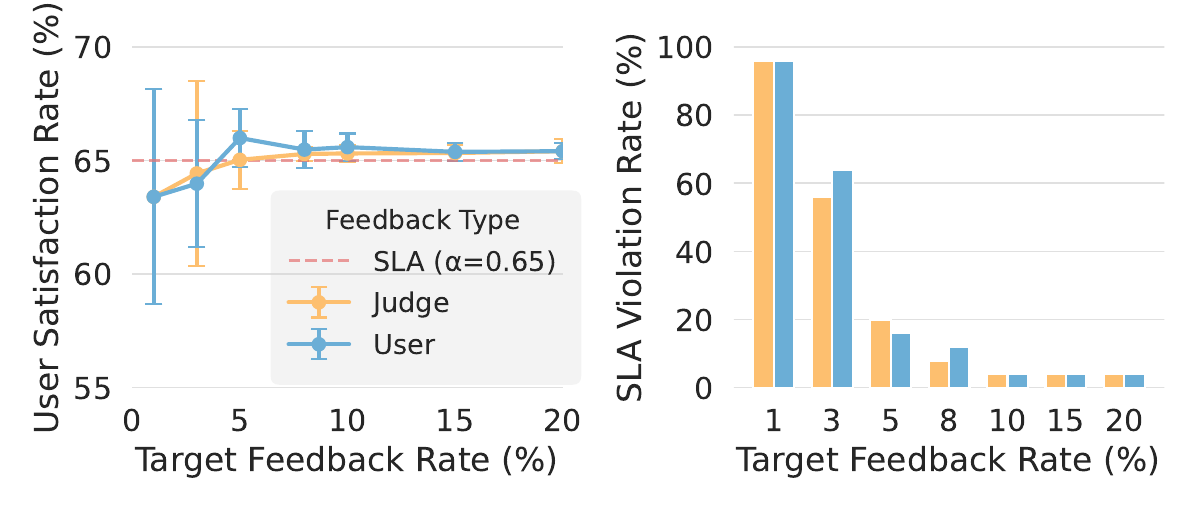}
        \vspace{-1.5em}
        \caption{\textbf{ACPBench} (1800 user requests)}
        \label{fig:acpbench_judge_feedback}
    \end{subfigure}
    \hfill
    \begin{subfigure}{0.49\textwidth}
        \centering
        \includegraphics[width=\linewidth]{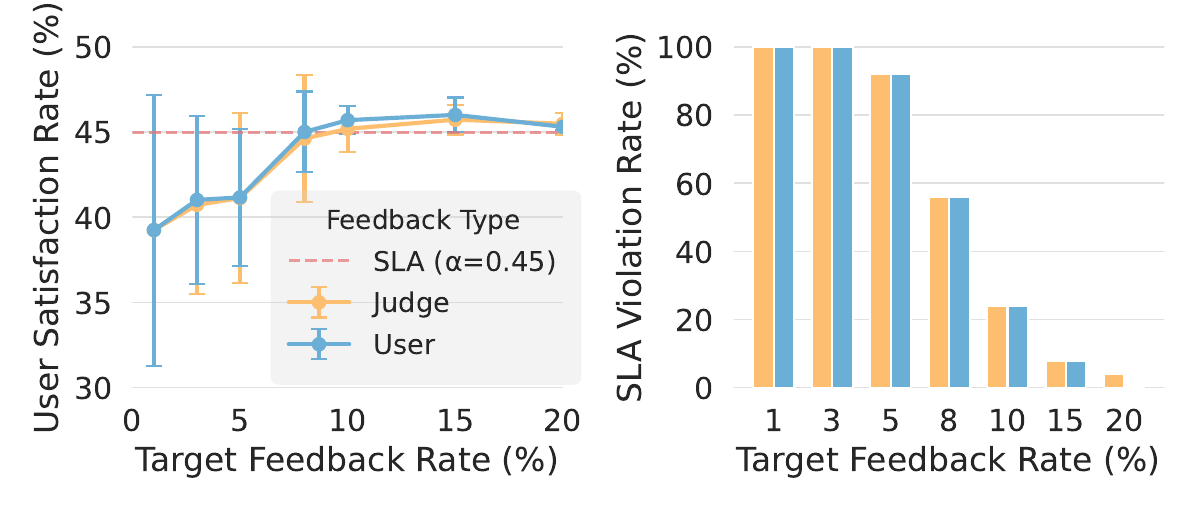}
        \vspace{-1.5em}
        \caption{\textbf{GPQA} (448 user requests)}
        \label{fig:gpqa_judge_feedback}
    \end{subfigure}
    \caption{Deployment of an LLM-based judge (GPT-5.4 with reasoning) to supplement predictor training under sparse user feedback. 
    Judge feedback closely aligns with user feedback and helps reduce SLA violations, enabling reliable routing without the need for dense user engagement.}
    \label{fig:judge_experiments}
    \vspace{1em}
\end{figure}

\section{Conclusion}

We present \algname, an online LLM router that extends the Lyapunov drift-plus-penalty framework, making it the first routing method to provide theoretical guarantees for cost optimality at a given user satisfaction rate under sparse, one-sided observational feedback. 
On average, SLARouter reduces operating costs by up to 2.22× over existing baselines while satisfying SLA constraints without per-benchmark tuning. 
When user feedback is sparse, an LLM-based judge provides a reliable synthetic supplement, enabling robust predictor updates without additional burden on end users. 
Our approach assumes constant feedback rates and IID request arrivals, which may not hold in production systems where feedback can arrive in bursts and where requests by individual users may be correlated.
Hence, we see priority in addressing non-constant feedback rates, feedback arriving in temporal waves, and the challenge of learning routing policies from negativity-biased feedback in future work.

\bibliography{main}
\bibliographystyle{abbrvnat}

\clearpage
\newpage
\appendix

\begin{center}
    {\bf\Large Appendix}
\end{center}

\startcontents[sections]
\printcontents[sections]{l}{1}{\setcounter{tocdepth}{2}}
\newpage

\setcounter{section}{0}
\renewcommand\thesection{\Alph{section}}
\numberwithin{equation}{section}
\counterwithin{figure}{section}
\counterwithin{algocf}{section}
\counterwithin{theorem}{section}
\counterwithin{lemma}{section}
\counterwithin{remark}{section}

\section{Proofs}
\label{appendix:proofs}

\subsection{Proof of \Cref{lemma:sgd_bound}}

For the loss gradient on a single sample, the variance term (stochastic noise) can be defined as
\begin{align}
    & \Expectbracket{\normsq{\nabla f(\mathbf{z}, \mathbf{a}, m) - \nabla F(\mathbf{z})} \bigg|\mathbf{z}} \leq \sigma^2, \ 
    \forall \mathbf{z}, m,\nonumber
\end{align}
where the inequality is due to \Cref{assumption:predictor_loss}.
Subsequently, we can define the variance term for the stochastic gradient (updated on mini-batch size $B$, where we use the short-hand notation $\mathbf{g}:=\nabla F(\mathbf{z}, \mathcal{B})$) and the empirical gradient $\nabla \hat{F}$ computed on $|\mathcal{D}|$ available samples as
\begin{align}
    & \Expectbracket{\normsq{\mathbf{g} - \nabla F(\mathbf{z})} \bigg|\mathbf{z}} \leq \frac{\sigma^2}{B}, \ 
    \forall \mathbf{z}, m,
\end{align}
\begin{align}
    & \Expectbracket{\normsq{\nabla \hat{F}(\mathbf{z}) - \nabla F(\mathbf{z})} \bigg|\mathbf{z}} \leq \frac{\sigma^2}{|\mathcal{D}|}, \ 
    \forall \mathbf{z}, m,
\end{align}
respectively.

We then have
\begin{align}
    &\Expectbracket{\normsq{\mathbf{g} - \nabla \hat{F}(\mathbf{z})} \bigg|\mathbf{z}} \nonumber \\
    &= \Expectbracket{\normsq{\mathbf{g} - \nabla \hat{F}(\mathbf{z}) + \nabla F(\mathbf{z}) - \nabla F(\mathbf{z})} \bigg|\mathbf{z}}\nonumber \\
    & \leq \Expectbracket{2\normsq{\mathbf{g}  - \nabla F(\mathbf{z})} + 2\normsq{\nabla \hat{F}(\mathbf{z}) - \nabla F(\mathbf{z})} \bigg|\mathbf{z}}\nonumber \\
    & \leq \frac{2\sigma^2}{B} + \frac{2\sigma^2}{|\mathcal{D}|}, \ 
    \forall \mathbf{z}, m
\end{align}

We note that $\Expectcond{ \mathbf{g}_k}{\mathbf{z}_k} := \nabla \hat{F}(\mathbf{z}_{k})$, which is generally not equal to $\nabla F(\mathbf{z}_k)$ under partial feedback.

\begin{proof}

Let $k$ denote the index of exploration steps. The SGD update of predictor training is
\begin{align}
\mathbf{z}_{k+1} \leftarrow \mathbf{z}_{k} - \eta_k \mathbf{g_k}
\end{align}

By $L$-smoothness, we have
\begin{align}
&\Expectcond{F(\mathbf{z}_{k+1})}{\mathbf{z}_k}\nonumber\\
&\leq F(\mathbf{z}_{k}) + \innerprod{\nabla F(\mathbf{z}_{k}), \Expectcond{\mathbf{z}_{k+1} - \mathbf{z}_k}{\mathbf{z}_k}} + \frac{L}{2}\normsq{\mathbf{z}_{k+1} - \mathbf{z}_k}\nonumber \\
&= F(\mathbf{z}_{k}) - \eta_k \innerprod{\nabla F(\mathbf{z}_{k}),\Expectcond{ \mathbf{g}_k}{\mathbf{z}_k} } + \frac{L\eta_k^2}{2}\Expectcond{\normsq{\mathbf{g}_k}}{\mathbf{z}_k}\nonumber \\
& \leq F(\mathbf{z}_{k}) - \eta_k \innerprod{\nabla F(\mathbf{z}_{k}),\nabla \hat{F}(\mathbf{z}_{k}) - \nabla F(\mathbf{z}_{k}) + \nabla F(\mathbf{z}_{k})} + \frac{L\eta_k^2}{2}\Expectcond{\normsq{\mathbf{g}_k}}{\mathbf{z}_k}\nonumber \\
& \leq F(\mathbf{z}_{k}) - \eta_k \innerprod{\nabla F(\mathbf{z}_{k}),\nabla \hat{F}(\mathbf{z}_{k}) - \nabla F(\mathbf{z}_{k}) } - \eta_k \normsq{\nabla F(\mathbf{z}_{k})} + \frac{L\eta_k^2}{2}\Expectcond{\normsq{\mathbf{g}_k}}{\mathbf{z}_k}\nonumber \\
& \leq F(\mathbf{z}_{k}) + \frac{\eta_k \normsq{\nabla F(\mathbf{z}_{k})}}{4} + \eta_k\normsq{\nabla \hat{F}(\mathbf{z}_{k}) - \nabla F(\mathbf{z}_{k}) } - \eta_k \normsq{\nabla F(\mathbf{z}_{k})} + \frac{L\eta_k^2}{2}\Expectcond{\normsq{\mathbf{g}_k}}{\mathbf{z}_k}\nonumber \\
&\leq F(\mathbf{z}_{k}) - \frac{3}{4}\eta_k \normsq{\nabla F(\mathbf{z}_{k})} + \eta_k\normsq{\nabla \hat{F}(\mathbf{z}_{k}) - \nabla F(\mathbf{z}_{k}) } + \frac{L\eta_k^2}{2}\left(\normsq{\nabla \hat{F}(\mathbf{z}_{k})} + \frac{2\sigma^2}{B} + \frac{2\sigma^2}{|\mathcal{D}|}\right)\nonumber \\
&\leq F(\mathbf{z}_{k}) - \frac{3}{4}\eta_k \normsq{\nabla F(\mathbf{z}_{k})} + \eta_k\normsq{\nabla \hat{F}(\mathbf{z}_{k}) - \nabla F(\mathbf{z}_{k}) } \nonumber \\
&\quad\quad + \frac{L\eta_k^2}{2}\left(\normsq{\nabla \hat{F}(\mathbf{z}_{k}) - \nabla F(\mathbf{z}_{k}) + \nabla F(\mathbf{z}_{k})} + \frac{2\sigma^2}{B} + \frac{2\sigma^2}{|\mathcal{D}|}\right)\nonumber \\
&\leq F(\mathbf{z}_{k}) - \frac{3}{4}\eta_k \normsq{\nabla F(\mathbf{z}_{k})} + \eta_k\normsq{\nabla \hat{F}(\mathbf{z}_{k}) - \nabla F(\mathbf{z}_{k}) } \nonumber \\
&\quad\quad + L\eta_k^2\left(\normsq{\nabla \hat{F}(\mathbf{z}_{k}) - \nabla F(\mathbf{z}_{k})} + \normsq{\nabla F(\mathbf{z}_{k})} + \frac{\sigma^2}{B} + \frac{\sigma^2}{|\mathcal{D}|}\right)\nonumber \\
&= F(\mathbf{z}_{k}) -\left( \frac{3}{4}\eta_k - L\eta_k^2 \right) \normsq{\nabla F(\mathbf{z}_{k})} + (\eta_k + L\eta_k^2) \normsq{\nabla \hat{F}(\mathbf{z}_{k}) - \nabla F(\mathbf{z}_{k})} + L \eta_k^2 \left(\frac{\sigma^2}{B} + \frac{\sigma^2}{|\mathcal{D}|} \right) \nonumber \\
&= F(\mathbf{z}_{k}) -\left( \frac{3}{4}\eta_k - L\eta_k^2 \right) \normsq{\nabla F(\mathbf{z}_{k})} + (\eta_k + 2L\eta_k^2) \frac{\sigma^2}{|\mathcal{D}|} + L \eta_k^2 \frac{\sigma^2}{B}.
\label{eq:proof_sgd_1}
\end{align}

From the PL condition, we have $\normsq{\nabla F(\mathbf{z})} \geq 2\mu\left(F(\mathbf{z})-F_\mathrm{min}\right), \forall \mathbf{z}$. Subtracting $F_\mathrm{min}$ on both sides of \eqref{eq:proof_sgd_1} and plugging in the PL inequality gives
\begin{align}
    &\Expectcond{F(\mathbf{z}_{k+1}) - F_\mathrm{min}}{\mathbf{z}_k} \nonumber \\ 
    & \leq  (F(\mathbf{z}_{k}) - F_\mathrm{min}) -2\mu\left( \frac{3}{4}\eta_k - L\eta_k^2 \right) \left(F(\mathbf{z}_k)-F_\mathrm{min}\right) + (\eta_k + 2L\eta_k^2) \frac{\sigma^2}{|\mathcal{D}|} + L \eta_k^2 \frac{\sigma^2}{B} \nonumber \\
     & =  \left(1 -2\mu \eta_k \left(\frac{3}{4}  - L\eta_k \right)\right) \left(F(\mathbf{z}_k)-F_\mathrm{min}\right) + (\eta_k + 2L\eta_k^2) \frac{\sigma^2}{|\mathcal{D}|} + L \eta_k^2 \frac{\sigma^2}{B}.
    \label{eq:proof_sgd_2}
\end{align}

Let $\eta_k \leq \frac{1}{2L}$. We have $\frac{3}{4} - L\eta_k \geq \frac{1}{4}$, and thus
\begin{align}
\Expectcond{F(\mathbf{z}_{k+1}) - F_\mathrm{min}}{\mathbf{z}_k} & \leq  
 \left(1 - \frac{\mu \eta_k}{2} \right) \left(F(\mathbf{z}_k)-F_\mathrm{min}\right) + \frac{2\eta_k\sigma^2}{|\mathcal{D}|} + \frac{L\eta_k^2\sigma^2}{B},
\label{eq:proof_sgd_3}
\end{align}
where we have absorbed the $2L\eta_k^2$ term in the $|\mathcal{D}|$-variance into the leading $\eta_k$ term since $\eta_k \leq \frac{1}{2L}$ implies $2L\eta_k^2 \leq \eta_k$. The remaining mini-batch variance term $\frac{L\eta_k^2\sigma^2}{B}$ is kept unchanged.

Taking total expectation gives
\begin{align}
    \Expectbracket{F(\mathbf{z}_{k+1})} - F_\mathrm{min} & \leq  
 \left(1 - \frac{\mu \eta_k}{2} \right) \left(\Expectbracket{F(\mathbf{z}_k)}-F_\mathrm{min}\right) + \frac{2\eta_k\sigma^2}{|\mathcal{D}|} + \frac{L\eta_k^2\sigma^2}{B}.
 \label{eq:proof_sgd_original_iterative_form}
\end{align}

Let $\mathcal{F}_k := \Expectbracket{F(\mathbf{z}_k)}-F_\mathrm{min}$. We have
\begin{align*}
    \mathcal{F}_2 &\leq \left(1-\frac{\mu\eta_1}{2}\right) \mathcal{F}_1 + \frac{2\eta_1\sigma^2}{|\mathcal{D}|} + \frac{L\eta_1^2\sigma^2}{B}.
\end{align*}
Next, we show by induction that
\begin{align}
    \mathcal{F}_k &\leq  \mathcal{F}_1 \prod_{\kappa=1}^{k-1}\left(1-\frac{\mu\eta_\kappa}{2}\right)  + \frac{2\sigma^2}{|\mathcal{D}|}\sum_{\kappa=1}^{k-1}\eta_\kappa\prod_{\kappa'=\kappa+1}^{k-1}\left(1-\frac{\mu\eta_{\kappa'}}{2}\right) + \frac{L\sigma^2}{B}\sum_{\kappa=1}^{k-1}\eta_\kappa^2\prod_{\kappa'=\kappa+1}^{k-1}\left(1-\frac{\mu\eta_{\kappa'}}{2}\right).
    \label{eq:proof_sgd_product_form}
\end{align}
To see this, we assume that \eqref{eq:proof_sgd_product_form} holds for $k\geq 2$, as we have shown above that it holds for $k=2$. We then have according to \eqref{eq:proof_sgd_original_iterative_form} that
\begin{align*}
    &\mathcal{F}_{k+1} \\
    &\leq \left(1-\frac{\mu\eta_k}{2}\right) \mathcal{F}_k + \frac{2\eta_k\sigma^2}{|\mathcal{D}|} + \frac{L\eta_k^2\sigma^2}{B} \\
    & \leq \left(1\!-\!\frac{\mu\eta_k}{2}\right)\!\!\left[ \mathcal{F}_1 \prod_{\kappa=1}^{k-1}\!\left(1\!-\!\frac{\mu\eta_\kappa}{2}\right)  + \frac{2\sigma^2}{|\mathcal{D}|}\sum_{\kappa=1}^{k-1}\eta_\kappa\!\!\prod_{\kappa'=\kappa+1}^{k-1}\!\!\left(1\!-\!\frac{\mu\eta_{\kappa'}}{2}\right) + \frac{L\sigma^2}{B}\sum_{\kappa=1}^{k-1}\eta_\kappa^2\!\!\prod_{\kappa'=\kappa+1}^{k-1}\!\!\left(1\!-\!\frac{\mu\eta_{\kappa'}}{2}\right) \right] \\
    &\quad\quad + \frac{2\eta_k\sigma^2}{|\mathcal{D}|} + \frac{L\eta_k^2\sigma^2}{B}\\
    & =  \mathcal{F}_1 \prod_{\kappa=1}^{k}\left(1-\frac{\mu\eta_\kappa}{2}\right)  + \frac{2\sigma^2}{|\mathcal{D}|}\sum_{\kappa=1}^{k-1}\eta_\kappa\prod_{\kappa'=\kappa+1}^{k}\left(1-\frac{\mu\eta_{\kappa'}}{2}\right) + \frac{L\sigma^2}{B}\sum_{\kappa=1}^{k-1}\eta_\kappa^2\prod_{\kappa'=\kappa+1}^{k}\left(1-\frac{\mu\eta_{\kappa'}}{2}\right)  \\
    &\quad\quad + \frac{2\eta_k\sigma^2}{|\mathcal{D}|} + \frac{L\eta_k^2\sigma^2}{B}\\  & =  \mathcal{F}_1 \prod_{\kappa=1}^{k}\left(1-\frac{\mu\eta_\kappa}{2}\right)  + \frac{2\sigma^2}{|\mathcal{D}|}\sum_{\kappa=1}^{k}\eta_\kappa\prod_{\kappa'=\kappa+1}^{k}\left(1-\frac{\mu\eta_{\kappa'}}{2}\right) + \frac{L\sigma^2}{B}\sum_{\kappa=1}^{k}\eta_\kappa^2\prod_{\kappa'=\kappa+1}^{k}\left(1-\frac{\mu\eta_{\kappa'}}{2}\right),
\end{align*}
where the last equality is because $\prod_{\kappa'=k+1}^{k}\left(1-\frac{\mu\eta_{\kappa'}}{2}\right)=1$. This proves \eqref{eq:proof_sgd_product_form}.

Recall that $\eta_k=\min\left\{\frac{4}{\mu (k+1)}, \frac{1}{2L}\right\}$. We note that, due to 
$L$-smoothness, we have $\mu\leq L$, because otherwise the $L$-smoothness contradicts with the 
PL condition. Therefore, $\frac{8L}{\mu} \geq 8$. We have $\eta_k=\frac{1}{2L}$ when 
$k\leq \tilde{k} := \left\lfloor \frac{8L}{\mu}-1\right\rfloor$. For $k>\tilde{k}$, we have 
$\eta_k=\frac{4}{\mu (k+1)}$, and in this case, $1-\frac{\mu\eta_k}{2} = 1 - \frac{2}{k+1} = 
\frac{k-1}{k+1}$.

We now bound the last two terms of \eqref{eq:proof_sgd_product_form}. Define
\begin{align}
    G_1(k) &:= \sum_{\kappa=1}^{k-1}\eta_\kappa\prod_{\kappa'=\kappa+1}^{k-1}\left(1-\frac{\mu\eta_{\kappa'}}{2}\right), \label{eq:proof_sgd_recurrence_eta}\\
    G_2(k) &:= \sum_{\kappa=1}^{k-1}\eta_\kappa^2\prod_{\kappa'=\kappa+1}^{k-1}\left(1-\frac{\mu\eta_{\kappa'}}{2}\right), \label{eq:proof_sgd_recurrence_eta_square}
\end{align}
so that the last two terms of \eqref{eq:proof_sgd_product_form} are equal to $\frac{2\sigma^2}{|\mathcal{D}|}G_1(k) + 
\frac{L\sigma^2}{B}G_2(k)$.

For the ease of presentation, we also define $k_0 := \tilde{k}+1$. Then, for $k\leq k_0$, we have $\kappa\leq \tilde{k}$ and $\kappa'\leq \tilde{k}$ in the sum and product terms, respectively.

\textbf{Bounding $G_1(k)$.}
For $k\leq k_0$, since $\eta_\kappa = \frac{1}{2L}$ and $1 - \frac{\mu\eta_{\kappa'}}{2} = 
1-\frac{\mu}{4L}$ throughout,
\begin{align}
    G_1(k) = \sum_{\kappa=1}^{k-1}\frac{1}{2L}\left(1-\frac{\mu}{4L}\right)^{k-1-\kappa} 
    \leq \frac{1}{2L}\cdot\frac{1}{\frac{\mu}{4L}} = \frac{2}{\mu}. \label{eq:proof_sgd_G1_small_k}
\end{align}

For $k\geq k_0$, the recurrence form of \eqref{eq:proof_sgd_recurrence_eta} is
\begin{align}
    G_1(k+1) &= \left(1-\frac{\mu\eta_k}{2}\right)G_1(k) + \eta_k \nonumber \\
    &= \frac{k-1}{k+1}\cdot G_1(k) + \frac{4}{\mu(k+1)}.
    \label{eq:proof_sgd_G1_recurrence}
\end{align}

We now show by induction that, for general $k > k_0$, we have
\begin{align}
    G_1(k) &\leq \frac{2\tilde{k}k_0}{\mu(k-1)k} + \frac{4(k-k_0)}{\mu k} \leq \mathcal{O}(1),
    \label{eq:proof_sgd_G1}
\end{align}
where the last inequality is straightforward because $\tilde{k}$ and $k_0$ are constants depending only on 
$L$ and $\mu$, and the last term of $\frac{2(k-k_0)}{\mu k}$ does not vanish but is upper bounded by $\frac{2}{\mu}$ for arbitrarily large $k$. We therefore focus on proving the first inequality as follows.

To see this, we first note that \eqref{eq:proof_sgd_G1} holds for $k=k_0$, because
\begin{align*}
    \frac{2\tilde{k}k_0}{\mu(k_0-1)k_0}+ \frac{4(k_0-k_0)}{\mu k_0} = \frac{2\tilde{k}k_0}{\mu\tilde{k}k_0} = \frac{2}{\mu} \geq G_1(k_0)
\end{align*}
according to \eqref{eq:proof_sgd_G1_small_k}. Now, assume that \eqref{eq:proof_sgd_G1} holds for some $k\geq k_0$. We then have from \eqref{eq:proof_sgd_G1_recurrence} that
\begin{align*}
    G_1(k+1) &= \frac{k-1}{k+1}\cdot G_1(k) + \frac{4}{\mu(k+1)} \\
    &\leq \frac{k-1}{k+1}\cdot \left[\frac{2\tilde{k}k_0}{\mu(k-1)k} + \frac{4(k-k_0)}{\mu k}\right] + \frac{4}{\mu(k+1)}\\
    &= \frac{2\tilde{k}k_0}{\mu k(k+1)} + \frac{4(k-k_0)(k-1)}{\mu k(k+1)} + \frac{4}{\mu(k+1)}\\
    &\leq \frac{2\tilde{k}k_0}{\mu k(k+1)} + \frac{4(k-k_0)}{\mu (k+1)} + \frac{4}{\mu(k+1)}\\
    &= \frac{2\tilde{k}k_0}{\mu k(k+1)} + \frac{4(k+1-k_0)}{\mu (k+1)},
\end{align*}
which proves \eqref{eq:proof_sgd_G1}.

\textbf{Bounding $G_2(k)$.}
For $k\leq k_0$, since $\eta_\kappa = \frac{1}{2L}$ and $1 - \frac{\mu\eta_{\kappa'}}{2} = 
1-\frac{\mu}{4L}$ throughout,
\begin{align}
    G_2(k) = \sum_{\kappa=1}^{k-1}\frac{1}{4L^2}\left(1-\frac{\mu}{4L}\right)^{k-1-\kappa}
    \leq \frac{1}{4L^2}\cdot\frac{1}{\frac{\mu}{4L}} = \frac{1}{L\mu}.
    \label{eq:proof_sgd_G2_small_k}
\end{align}

For $k\geq k_0$, the recurrence form of \eqref{eq:proof_sgd_recurrence_eta_square} is
\begin{align}
    G_2(k+1) &= \left(1-\frac{\mu\eta_k}{2}\right)G_2(k) + \eta_k^2 \nonumber \\
    &= \frac{k-1}{k+1}\cdot G_2(k) + \frac{16}{\mu^2(k+1)^2}.
    \label{eq:proof_sgd_G2_recurrence}
\end{align}

We now show by induction that, for general $k > k_0$, we have
\begin{align}
    G_2(k) &\leq \frac{\tilde{k}k_0}{L\mu(k-1)k} + \frac{16}{\mu^2 k} = \mathcal{O}\left(\frac{1}{k}\right),
    \label{eq:proof_sgd_G2}
\end{align}
where the last equality is straightforward because $\tilde{k}$ and $k_0$ are constants depending only on 
$L$ and $\mu$.  We therefore focus on proving the first inequality as follows.

To see this, we first note that \eqref{eq:proof_sgd_G2} holds for $k=k_0$, because
\begin{align*}
    \frac{\tilde{k}k_0}{L\mu(k_0-1)k_0} + \frac{16}{\mu^2 k} = \frac{\tilde{k}k_0}{L\mu\tilde{k}k_0} + \frac{16}{\mu^2 k} = \frac{1}{L\mu} + \frac{16}{\mu^2 k} \geq \frac{1}{L\mu} \geq G_2(k_0),
\end{align*}
where the last inequality follows from \eqref{eq:proof_sgd_G2_small_k}.
Now, assume that \eqref{eq:proof_sgd_G2} holds for some $k\geq k_0$. We then have from \eqref{eq:proof_sgd_G2_recurrence} that
\begin{align*}
    G_2(k+1) 
    &= \frac{k-1}{k+1}\cdot G_2(k) + \frac{16}{\mu^2(k+1)^2} \\
    &\leq \frac{k-1}{k+1}\cdot \left[ \frac{\tilde{k}k_0}{L\mu(k-1)k} + \frac{16}{\mu^2 k}\right] + \frac{16}{\mu^2(k+1)^2}\\
    &= \frac{\tilde{k}k_0}{L\mu k(k+1)} + \frac{16(k-1)}{\mu^2 k(k+1)} + \frac{16}{\mu^2(k+1)^2}\\
    &\leq \frac{\tilde{k}k_0}{L\mu k(k+1)} + \frac{16(k-1)}{\mu^2 k(k+1)} + \frac{16}{\mu^2k(k+1)}\\
    &\leq \frac{\tilde{k}k_0}{L\mu k(k+1)} + \frac{16}{\mu^2(k+1)},
\end{align*}
which proves \eqref{eq:proof_sgd_G2}.

\textbf{Total Noise Term.}
Combining \eqref{eq:proof_sgd_G1} and \eqref{eq:proof_sgd_G2} with 
\eqref{eq:proof_sgd_product_form}, the total noise contribution satisfies
\begin{align}
    \frac{2\sigma^2}{|\mathcal{D}|}\,G_1(k) + \frac{\sigma^2}{B}\,G_2(k) 
    = \mathcal{O}\left(\frac{\sigma^2}{|\mathcal{D}|} + \frac{\sigma^2}{kB}\right).
    \label{eq:proof_sgd_noise_terms}
\end{align}

\textbf{Bounding First Term of \eqref{eq:proof_sgd_product_form}.}
Finally, we consider the first term of \eqref{eq:proof_sgd_product_form}. When $k\leq k_0$,
\begin{align}
    H(k):= \prod_{\kappa=1}^{k-1}\left(1-\frac{\mu\eta_\kappa}{2}\right) = \left(1-\frac{\mu}{4L}\right)^{k-1}.
    \label{eq:proof_sgd_first_term_small_k}
\end{align}

For $k\geq k_0$, recall that $\eta_k=\frac{4}{\mu (k+1)}$ and thus $1-\frac{\mu\eta_k}{2} = 1 - \frac{2}{k+1} = 
\frac{k-1}{k+1}$. We have the following recurrence relation:
\begin{align}
    H(k+1) &= \left(1-\frac{\mu\eta_k}{2}\right)H(k) = \frac{k-1}{k+1}\cdot H(k). 
    \label{eq:proof_sgd_first_term_recurrence}
\end{align}

We now prove by induction that, for general $k> k_0$,
\begin{align}
    H(k) = \left(1-\frac{\mu}{4L}\right)^{\tilde{k}} \cdot\frac{\tilde{k}k_0}{(k-1)k} =\mathcal{O}\left(\frac{1}{k^2}\right),
    \label{eq:proof_sgd_first_term}
\end{align}
where the last equality is because $\tilde{k}$ and $k_0$ are constants as they only depend on $L$ and $\mu$. Thus, we only focus on proving the first equality.

We note that \eqref{eq:proof_sgd_first_term} holds for $k=k_0$, because
\begin{align*}
    \left(1-\frac{\mu}{4L}\right)^{\tilde{k}} \cdot\frac{\tilde{k}k_0}{\tilde{k}k_0} = \left(1-\frac{\mu}{4L}\right)^{k_0-1} = H(k_0),
\end{align*}
where the last equality is from \eqref{eq:proof_sgd_first_term_small_k}. Now, assume that \eqref{eq:proof_sgd_first_term} holds for some $k\geq k_0$. We then have from \eqref{eq:proof_sgd_first_term_recurrence} that
\begin{align*}
    H(k+1) &= \frac{k-1}{k+1}\cdot H(k)\\
    &= \frac{k-1}{k+1}\cdot \left(1-\frac{\mu}{4L}\right)^{\tilde{k}} \cdot\frac{\tilde{k}k_0}{(k-1)k} \\
    &= \left(1-\frac{\mu}{4L}\right)^{\tilde{k}} \cdot\frac{\tilde{k}k_0}{k(k+1)}, 
\end{align*}
which proves \eqref{eq:proof_sgd_first_term}.

\textbf{Final Bound.}
Combining \eqref{eq:proof_sgd_noise_terms} and \eqref{eq:proof_sgd_first_term} with \eqref{eq:proof_sgd_product_form}, we obtain
\begin{align}
    \mathcal{F}_k = \Expectbracket{F(\mathbf{z}_k)}-F_\mathrm{min}\leq \mathcal{O}\left(\frac{1}{k^2} + \frac{\sigma^2}{|\mathcal{D}|} + \frac{\sigma^2}{kB}\right).
\end{align}

\end{proof}

\subsection{Total Expectation Result of \Cref{lemma:sgd_bound}}

We give a formal proof of the total expectation result of \Cref{lemma:sgd_bound}, starting with bounding the expectation of $k$.

\begin{lemma}
\label{lemma:k_bound}
Let $k$ denote the random variable of the number of predictor training steps after processing $t$ requests, and $\feedbackprob$ denote the fixed feedback availability rate across all steps. We have
\begin{align}
    \Expectbracket{k} \leq \mathcal{O}\left(\feedbackprob t^{\frac{3}{4}}\right) \textnormal{ and }
    \Expectbracket{\frac{1}{k}} \leq \mathcal{O}\left(\frac{1}{\feedbackprob t^{\frac{3}{4}}}\right).
\end{align}    
\end{lemma}

\begin{proof}
Recall that $X_t \sim \mathrm{Bernoulli}(p_t)$ is an indicator denoting whether an exploration step occurs when processing request $t$.
Furthermore, $Y_t \sim \text{Bernoulli}(\feedbackprob)$ is an indicator denoting whether the user provides feedback at a given exploration step, where $\feedbackprob$ is constant across all steps.
Assuming $k$ is the total number of SGD steps after processing $t$ requests, and noting that $X_\tau \perp Y_\tau$, we have
\begin{align}
    &\lambda := \Expectbracket{k} = \Expectbracket{\sum_{\tau=1}^t X_\tau Y_\tau} = \sum_{\tau=1}^t \Expectbracket{X_\tau} \Expectbracket{Y_\tau} = \sum_{\tau=1}^t p_\tau \feedbackprob \\
    &= \sum_{\tau=1}^t \feedbackprob \min\left(1, \frac{c}{\sqrt[4]{\tau}}\right) = \sum^{c^4}_{\tau=1} \feedbackprob + \sum^t_{\tau=c^4+1} \frac{\feedbackprob c}{\sqrt[4]{\tau}}\\
    &= c^4 \feedbackprob + \frac{4 \feedbackprob c}{3} (t^{\frac{3}{4}} - c^3) = \Theta\left(\feedbackprob t^{\frac{3}{4}}\right)
    \label{eq:proof_k_expectation}
\end{align}

This proves the first result.

Considering $\frac{1}{k}$, we adopt the convention that $\frac{1}{k} := 0$ when $k=0$, since no predictor update has occurred in that case and \Cref{lemma:sgd_bound} (which is stated for $k\geq 1$) contributes no term. Equivalently, all expectations involving $\frac{1}{k}$ are taken on the event $\{k\geq 1\}$. We then note that
\begin{align}
    \Expectbracket{\frac{1}{k}} = \Expectbracket{\frac{1}{k}\cdot\Identity_{k\geq\lambda/2}} + \Expectbracket{\frac{1}{k}\cdot\Identity_{k<\lambda/2}} \leq \frac{2}{\lambda} + \Expectbracket{\frac{1}{k}\cdot\Identity_{k<\lambda/2}},
    \label{eq:proof_k_decompose}
\end{align}
where $\Identity_\mathcal{C}$ is an indicator function of whether the condition $\mathcal{C}$ holds, and $\lambda$ is as derived in \eqref{eq:proof_k_expectation} accounting for the feedback probability $\feedbackprob$.

We now consider the last term in \eqref{eq:proof_k_decompose}.
The multiplicative Chernoff bound shows that
\begin{align*}
    \Pr\{k\leq(1-\delta)\lambda\}\leq e^{-\frac{\delta^2\lambda}{2}},
\end{align*}
for $0<\delta<1$. Choosing $\delta=\frac{1}{2}$ gives
\begin{align*}
    \Pr\left\{k\leq\frac{\lambda}{2}\right\}\leq e^{-\frac{\lambda}{8}}.
\end{align*}
Under the above convention (so that $\frac{1}{k}\leq 1$ on $\{k\geq 1\}$ and $\frac{1}{k}=0$ on $\{k=0\}$),
\begin{align}
    \Expectbracket{\frac{1}{k}\cdot\Identity_{k<\lambda/2}} \leq \Expectbracket{\Identity_{k<\lambda/2}} \leq e^{-\frac{\lambda}{8}} = \frac{1}{e^{\frac{\lambda}{8}}} \leq \frac{1}{1+\frac{\lambda}{8}},
    \label{eq:proof_k_decompose_second_term}
\end{align}

where the last inequality is due to the elementary relation that $e^x\geq 1+x$ for any $x$.
Combining \eqref{eq:proof_k_expectation}, \eqref{eq:proof_k_decompose}, and \eqref{eq:proof_k_decompose_second_term}, we obtain
\begin{align}
    \Expectbracket{\frac{1}{k}} \leq \mathcal{O}\left(\frac{1}{\feedbackprob t^{\frac{3}{4}}}\right).
\end{align}

\end{proof}

We then have the following lemma.

\begin{lemma}[Formal total expectation bound of \Cref{lemma:sgd_bound}]
\label{lemma:sgd_bound_total_expectation}
Under the same conditions as for \Cref{lemma:sgd_bound}, we have
\begin{align}
\Expectbracket{F(\mathbf{z}_k)}-F_\mathrm{min}\leq \mathcal{O}\left(\frac{1+\sigma^2}{t^{\sfrac{3}{4}}\phi}\right).        
\end{align}
\end{lemma}
\begin{proof}
    The result follows directly by noting $k = |\mathcal{D}|$ and taking total expectation to the bound in \Cref{lemma:sgd_bound}.
\end{proof}

\subsection{Proof of \Cref{lemma:satisfaction_error_bound_with_delta}}

We first prove the expectation version of \Cref{lemma:satisfaction_error_bound_with_delta}.

For the ease of discussion, let $\epsilon$ denote the upper bound of 
$\Expectbracket{F(\mathbf{z}_k)}-F_\mathrm{min}$ so that 
$\epsilon=\mathcal{O}\left(\frac{1+\sigma^2}{t^{3/4}\phi}\right)$
according to \Cref{lemma:sgd_bound_total_expectation}. We consider $\beta=M$ in this proof.

When $\Expectbracket{F(\mathbf{z}_k)}-F_\mathrm{min}\leq \epsilon$, from \eqref{eq:cross-entropy} and noting that the cross entropy is non-negative, we have
\begin{align}
    &\Expectbracket{\max_m \left\{ -s_{m,t}\log  \hat{s}_{m,t} - (1- s_{m,t})  
        \log (1- \hat{s}_{m,t} )\right\} }\nonumber \\
    &\leq \Expectbracket{\sum_{m=1}^M \left( -s_{m,t}\log  \hat{s}_{m,t} - (1- s_{m,t})  
        \log (1- \hat{s}_{m,t} ) \right) }\nonumber \\
    &\leq M \Expectbracket{F(\mathbf{z}_k)}\nonumber \\
    &\leq M F_\mathrm{min} + \epsilon M. 
    \label{proof_lemma:absolute_difference_bound_1}
\end{align}
For arbitrary sample $\mathbf{a}_t$ and model $m$, let
\begin{align*}
    \Gamma := -  s_{m,t}\log \hat{s}_{m,t}  - (1- s_{m,t})  \log (1-\hat{s}_{m,t}).
\end{align*}

We consider two cases as follows.
When $s_{m,t}=0$, we have
\begin{align*}
    &\Gamma = -\log (1-\hat{s}_{m,t}) \\
    &\Rightarrow \hat{s}_{m,t} = 1-e^{-\Gamma}\\
    &\Rightarrow |\hat{s}_{m,t} - s_{m,t}| = 1-e^{-\Gamma}.
\end{align*}
When $s_{m,t}=1$, we have
\begin{align*}
    &\Gamma = -\log \hat{s}_{m,t} \\
    &\Rightarrow 1-\hat{s}_{m,t} = 1-e^{-\Gamma}\\
    &\Rightarrow |\hat{s}_{m,t} - s_{m,t}| = 1-e^{-\Gamma}.
\end{align*}
Noting the elementary inequality $e^{-\Gamma}\geq 1-\Gamma$, we obtain
\begin{align*}
    |\hat{s}_{m,t} - s_{m,t}| \leq \Gamma.
\end{align*}

Because this relation holds for any sample and the corresponding $\Gamma$ defined on the 
sample, the expectation of the cross-entropy loss cannot be smaller than the expectation 
of the absolute difference. Combining with \eqref{proof_lemma:absolute_difference_bound_1}, 
we have
\begin{align}
    \Expectbracket{\max_m|\hat{s}_{m,t} - s_{m,t}|} 
    &\leq M F_\mathrm{min} + \epsilon M \nonumber\\
    &= M F_\mathrm{min}
    + \mathcal{O}\left(\frac{M(1+\sigma^2)}{t^{3/4}\phi}\right).
    \label{eq:abs_diff_expectation}
\end{align}
The final result then follows by setting $\beta=M$.

Next, we proceed with the proof of the high probability version of \Cref{lemma:satisfaction_error_bound_with_delta}.
Applying Markov's inequality to $F(\mathbf{z}_k) - F_\mathrm{min} \geq 0$, we have
\begin{align}
    \Pr\left\{F(\mathbf{z}_k) - F_\mathrm{min} \geq 
    \frac{\Expectbracket{F(\mathbf{z}_k)} - F_\mathrm{min}}{\delta}\right\}\leq \delta,
\end{align}
so with probability at least $1-\delta$,
\begin{align}
    F(\mathbf{z}_k) - F_\mathrm{min} \leq
    \frac{\Expectbracket{F(\mathbf{z}_k)} - F_\mathrm{min}}{\delta}
    \leq \frac{\epsilon}{\delta}
    = \mathcal{O}\left(\frac{1+\sigma^2}{ t^{\sfrac{3}{4}} \feedbackprob \delta}\right).
    \label{eq:markov_application}
\end{align}

Finally, using the pointwise relation $|\hat{s}_{m,t} - s_{m,t}| \leq \Gamma$ established above and $\max_m|\hat{s}_{m,t} - s_{m,t}| \leq \sum_{m=1}^M \Gamma \leq M F(\mathbf{z}_k)$ (with the cross-entropy evaluated on request $t$, analogous to the expectation version), the high-probability bound \eqref{eq:markov_application} on $F(\mathbf{z}_k) - F_\mathrm{min}$ gives, with probability at least $1-\delta$,
\begin{align}
    \max_m|\hat{s}_{m,t} - s_{m,t}| \leq M F_\mathrm{min} + M\left(F(\mathbf{z}_k) - F_\mathrm{min}\right) \leq M F_\mathrm{min} + \mathcal{O}\left(\frac{M (1+\sigma^2)}{ t^{\sfrac{3}{4}} \feedbackprob \delta}\right),
\end{align}
After setting $\beta = M$, this completes the proof of \Cref{lemma:satisfaction_error_bound_with_delta}.

\subsection{Proof of \Cref{theorem:max_queue_length}}

Let $\rho := \frac{q-\alpha-(1-\feedbackprob)\beta F_\mathrm{min}}{2}$, which is positive by \Cref{assumption:queue_length}. We invoke \Cref{lemma:satisfaction_error_bound_with_delta} with failure probability $\delta := \rho$. Since $\rho \leq \frac{q-\alpha}{2} < 1$, this is a valid choice with $\delta\in(0,1)$. It gives
$$q-\delta-\alpha = q-\alpha-\rho = \frac{q-\alpha}{2} + \frac{(1-\feedbackprob)\beta F_\mathrm{min}}{2} > 0,$$
under \Cref{assumption:queue_length}, so $q-\delta = q-\rho > \alpha$.

With $\psi=\Theta\!\left(\left(\frac{\beta(1+\sigma^2)}{\feedbackprob\,\rho\left(\frac{1}{3}-\beta F_\mathrm{min}\right)}\right)^{\sfrac{4}{3}}\right)$ as defined in \Cref{theorem:max_queue_length}, every $t\geq\psi$ satisfies
\begin{align}
    \mathcal{O}\!\left(\frac{\beta(1+\sigma^2)}{t^{\sfrac{3}{4}}\feedbackprob\,\rho}\right) \leq \frac{1}{3}-\beta F_\mathrm{min}.
    \label{eq:psi_consequence}
\end{align}

Thus, for $t\geq\psi$,
the high-probability version of \Cref{lemma:satisfaction_error_bound_with_delta} with $\delta=\rho$ gives, with a probability of at least $1-\rho$,
\begin{align}
    \max_m|\hat{s}_{m,t}-s_{m,t}| \leq \beta F_\mathrm{min}+\mathcal{O}\!\left(\frac{\beta(1+\sigma^2)}{t^{\sfrac{3}{4}}\feedbackprob\,\delta}\right) = \beta F_\mathrm{min}+\mathcal{O}\!\left(\frac{\beta(1+\sigma^2)}{t^{\sfrac{3}{4}}\feedbackprob\,\rho}\right) \leq \frac{1}{3},
    \label{eq:proof_queue_length_max_s_estimate_bound}
\end{align}
where the last inequality uses \eqref{eq:psi_consequence}.
Therefore, with a probability of at least $1-\rho$, if $s_{m,t}=1$ we must have $\hat{s}_{m,t} \geq \frac{2}{3}$, and if $s_{m,t}=0$ we must have $\hat{s}_{m,t} \leq \frac{1}{3}$.

Using $1-\feedbackprob\leq1$ and \eqref{eq:psi_consequence}, we can also obtain the following for $t\geq\psi$:
\begin{align}
    (1-\feedbackprob)\mathcal{O}\!\left(\frac{\beta(1+\sigma^2)}{t^{\sfrac{3}{4}}\feedbackprob}\right) \leq \mathcal{O}\!\left(\frac{\beta(1+\sigma^2)}{t^{\sfrac{3}{4}}\feedbackprob}\right) \leq \rho\cdot\mathcal{O}\!\left(\frac{\beta(1+\sigma^2)}{t^{\sfrac{3}{4}}\feedbackprob\,\rho}\right) \leq \rho\left(\frac{1}{3}-\beta F_\mathrm{min}\right) \leq \frac{\rho}{2},
    \label{eq:proof_queue_length_rho_over_2_bound}
\end{align}
which will be used in the drift analysis later.

Let $\gamma := \max\left\{\alpha \psi; 3V\Delta_C\right\}$ and $Z_t := \max\{0, Q_t - \gamma\}$. For any given $t$, let $\tilde{m}$ denote the selected model obtained from the solution to \eqref{eq:per-request-optimization}, i.e., $y_{\tilde{m},t}=1$ and $y_{m',t}=0$ for $m'\neq \tilde{m}$.

\textbf{The queue cannot reach $\gamma$ before step $\psi$.}
The satisfaction signal driving the queue update \eqref{eq:queue-update} is in $[0,1]$ and $\alpha\in(0,1)$, so $Q_{t+1}\leq Q_t+\alpha$ holds for any $t$. With $Q_1=0$, this gives $Q_t\leq\alpha(t-1)$ for all $t$. Because $\gamma\geq\alpha\psi$, it follows that $Q_t\leq\alpha(t-1)<\alpha\psi\leq\gamma$ for every $t\leq\psi$, so $Z_t=0$ throughout $t\leq\psi$. Equivalently, $Z_t>0$ (i.e., $Q_t>\gamma$) can occur only when $t>\psi$. For this reason, in the above, we derived \eqref{eq:proof_queue_length_max_s_estimate_bound} and \eqref{eq:proof_queue_length_rho_over_2_bound} only for the case of $t\geq\psi$.

\textbf{Model selection under a large queue.}
We show that, when $Q_t > \gamma$, we have $\hat{s}_{\hat{m},t} - \hat{s}_{\tilde{m},t} <\frac{1}{3}$, where $\hat{m} := \arg\max_m \hat{s}_{m,t}$. The proof is by contradiction. Suppose $\hat{s}_{\hat{m},t} - \hat{s}_{\tilde{m},t} \geq \frac{1}{3}$, then choosing $\hat{m}$ instead of $\tilde{m}$, i.e., letting 
$y_{\hat{m},t}=1$ and $y_{\tilde{m},t}=0$, will reduce the second term of the objective \eqref{eq:per-request_obj}  by more than $\frac{Q_t}{3}$. The increase in the first term of  \eqref{eq:per-request_obj}  due to this change is at most $V\Delta_C$. Since $Q_t > \gamma \geq 3V\Delta_C$, we know that this alternative model choice decreases the objective \eqref{eq:per-request_obj}, which contradicts that $\tilde{m}$ is the optimal model choice from \eqref{eq:per-request-optimization}.

\textbf{The selected model is satisfactory with high probability.}
Consider any $t$ with $Q_t > \gamma$. Let $A:=\{\max_m|\hat{s}_{m,t} - s_{m,t}| \leq \frac{1}{3}\}$. Since $t>\psi$ when $Q_t > \gamma$, we have $\Pr\{A\}\geq 1-\delta$. By \Cref{assumption:queue_length}, there is a model $m^*$ with $\Pr\{s_{m^*,t}=1\}\geq q$. On event $A$, whenever $s_{m^*,t}=1$ we have $\hat{s}_{m^*,t}\geq\frac{2}{3}$, hence $\hat{s}_{\hat{m},t}\geq\frac{2}{3}$, and the selection bound above gives $\hat{s}_{\tilde{m},t} > \hat{s}_{\hat{m},t}-\frac{1}{3}\geq\frac{1}{3}$, which forces $s_{\tilde{m},t}=1$. Thus $\{s_{m^*,t}=1\}\cap A \subseteq \{s_{\tilde{m},t}=1\}$, and by the union bound
$$\Pr\{s_{\tilde{m},t} = 1\} \geq \Pr\{s_{m^*,t}=1\} - \Pr\{A^c\} \geq q - \delta.$$
Because requests are IID and $Z_t$ is past-measurable, the same argument holds after conditioning on $Z_t$, giving $\Expectcond{s_{\tilde{m},t}}{Z_t} = \Pr\{s_{\tilde{m},t}=1\mid Z_t\}\geq q-\delta$.

Let
\begin{equation}
u_{\tilde{m}, t} := \begin{cases}
    s_{\tilde{m}, t}, & \textrm{if } s_{\tilde{m}, t}\neq\emptyset\\
    \hat{s}_{\tilde{m}, t}, & \textrm{if } s_{\tilde{m}, t}=\emptyset
\end{cases}.
\label{eq:queue_update_u_cases}
\end{equation}

We also note that the queue arrival at step $t$ satisfies $\left(\alpha - u_{\tilde{m}, t}\right)^2 \leq 1$, since $\alpha\in(0,1)$ and $u_{\tilde{m},t}\in[0,1]$.

\textbf{Lyapunov drift.}
We now bound the drift of $Z_t$. If $Z_t = 0$, then $Z_{t+1} = \max\{0, (Q_t - \gamma) + \alpha - u_{\tilde{m},t}\} \leq \max\{0, \alpha - u_{\tilde{m},t}\} \leq 1$, so $\frac{1}{2}\Expectcond{Z^2_{t+1} - Z^2_t}{Z_t} \leq \frac{1}{2} = -\frac{\rho}{2} Z_t + \frac{1}{2}$ (as $Z_t = 0$). If instead $Z_t > 0$, then $Q_t > \gamma$ and $t > \psi$, so \eqref{eq:proof_queue_length_max_s_estimate_bound}, \eqref{eq:proof_queue_length_rho_over_2_bound},  and $\Expectcond{s_{\tilde{m},t}}{Z_t} \geq q-\delta$ all apply, and we obtain
\begin{align*}
    &\frac{1}{2}\Expectcond{Z^2_{t+1} - Z^2_t}{Z_t} \\
    & = \frac{1}{2} \Expectcond{\left(\max\{0, Z_t + \alpha - u_{\tilde{m},t}\}\right)^2 - Z^2_t}{Z_t} \\
    & \leq \frac{1}{2} \Expectcond{\left(Z_t + \alpha - u_{\tilde{m},t}\right)^2 - Z^2_t}{Z_t} \\
    & \leq \Expectcond{Z_t\left(\alpha - u_{\tilde{m},t}\right) + \frac{1}{2}\left(\alpha - u_{\tilde{m},t}\right)^2 }{Z_t} \\
    & \overset{(a)}{\leq} \Expectcond{Z_t\left(\alpha - s_{\tilde{m},t}\right) }{Z_t} + \Expectcond{Z_t \left(s_{\tilde{m},t} - u_{\tilde{m},t}\right) }{Z_t} + \frac{1}{2} \\
    & \overset{(b)}{\leq} \Expectcond{Z_t\left(\alpha - s_{\tilde{m},t}\right) }{Z_t} + (1-\feedbackprob)\, Z_t\, \Expectbracket{|\hat{s}_{\tilde{m},t} - s_{\tilde{m},t} |} + \frac{1}{2} \\
    & \overset{(c)}{\leq} Z_t\left(\alpha - (q-\delta)\right) + (1-\feedbackprob)\, Z_t\left(\beta F_\mathrm{min}
        + \mathcal{O}\left(\frac{\beta (1+\sigma^2)}{ t^{\sfrac{3}{4}} \feedbackprob}\right)\right) + \frac{1}{2}\\
    & \overset{(d)}{=} Z_t\left(-\rho
        + (1-\feedbackprob)\mathcal{O}\left(\frac{\beta (1+\sigma^2)}{ t^{\sfrac{3}{4}} \feedbackprob}\right)\right) +\frac{1}{2}\\
    & \overset{(e)}{\leq} -\frac{\rho}{2} Z_t +\frac{1}{2},
\end{align*}
where $(a)$ adds and subtracts $s_{\tilde{m},t}$ and uses $\frac{1}{2}(\alpha-u_{\tilde{m},t})^2\leq\frac{1}{2}$; $(b)$ uses $s_{\tilde{m},t} - u_{\tilde{m},t} = (1-Y_t)(s_{\tilde{m},t} - \hat{s}_{\tilde{m},t})$, where $Y_t\sim\mathrm{Bernoulli}(\feedbackprob)$ is the feedback indicator independent of the satisfaction outcomes (\Cref{assumption:partial_feedback}), so that $\Expectcond{Z_t(s_{\tilde{m},t} - u_{\tilde{m},t})}{Z_t} = (1-\feedbackprob)Z_t\,\Expectbracket{s_{\tilde{m},t} - \hat{s}_{\tilde{m},t}} \leq (1-\feedbackprob)Z_t\,\Expectbracket{|\hat{s}_{\tilde{m},t} - s_{\tilde{m},t}|}$; $(c)$ applies \Cref{lemma:satisfaction_error_bound_with_delta} (expectation version) together with $\Expectcond{s_{\tilde{m},t}}{Z_t} \geq q-\delta$; $(d)$ uses $\alpha-(q-\delta)+(1-\feedbackprob)\beta F_\mathrm{min} = -\rho$, which holds by the choice $\delta=\rho=\frac{q-\alpha-(1-\feedbackprob)\beta F_\mathrm{min}}{2}$; and $(e)$ holds because $t \geq \psi$ ensures $(1-\feedbackprob)\mathcal{O}\left(\frac{\beta (1+\sigma^2)}{ t^{\sfrac{3}{4}} \feedbackprob}\right) \leq \frac{\rho}{2}$ from \eqref{eq:proof_queue_length_rho_over_2_bound}.

Combining the two cases, the drift bound $\frac{1}{2}\Expectcond{Z^2_{t+1} - Z^2_t}{Z_t} \leq -\frac{\rho}{2} Z_t + \frac{1}{2}$ holds for all $t\geq 1$. Taking total expectation gives
\begin{align*}
    \Expectbracket{Z^2_{t+1} - Z^2_t}  &\leq -\rho\,\Expectbracket{Z_t} + 1.
\end{align*}

Since $Z_1=0$, telescoping gives
\begin{align}
    \Expectbracket{Z^2_{t}}  \leq (t-1) - \rho\sum_{\tau=1}^{t-1} \Expectbracket{Z_\tau},
    \label{eq:proof_queue_length_0}
\end{align}
for $t>1$.

Noting that $\rho > 0$ and $Z_\tau\geq 0$, we have $\Expectbracket{Z^2_{t}} \leq t-1$, so by Jensen's inequality,
\begin{align}
    \Expectbracket{Z_t} \leq \sqrt{\Expectbracket{Z^2_{t}} } \leq \sqrt{t-1} \leq \mathcal{O}\left( \sqrt{t}\right),
    \label{eq:proof_queue_length_1}
\end{align}
for $t\geq 1$.

In addition, from \eqref{eq:proof_queue_length_0} and $\Expectbracket{Z^2_{t}}\geq 0$, we have
\begin{align*}
    \rho\sum_{\tau=1}^{t-1} \Expectbracket{Z_\tau}  \leq (t-1) - \Expectbracket{Z^2_{t}} \leq t-1.
\end{align*}
Therefore,
\begin{align}
    \frac{1}{t}\sum_{\tau=1}^{t} \Expectbracket{Z_\tau}  \leq \frac{1}{t}\left(\frac{t-1}{\rho} + \Expectbracket{Z_t}\right) \leq \mathcal{O}\left(\frac{1}{\rho}\right),
    \label{eq:proof_queue_bound_average}
\end{align}
for $t\geq 1$.

The final result then follows by combining $Q_t \leq \gamma + Z_t$ with \eqref{eq:proof_queue_length_1} and \eqref{eq:proof_queue_bound_average}, giving the two bounds.
\qed

\subsection{Proof of \Cref{corollary:constraint_satisfaction_at_t}}

Using the definition in \eqref{eq:queue_update_u_cases}, the virtual queue update \eqref{eq:queue-update} can be written as $Q_{t+1} = \max\{0, Q_t + \alpha - u_{\tilde{m},t}\}$. We have $Q_{t+1}\geq Q_t + \alpha - u_{\tilde{m},t}$, i.e., $\alpha - u_{\tilde{m},t}\leq Q_{t+1}-Q_t$. Summing over $t=1,\ldots,T$ and using $Q_1=0$ gives
\begin{align}
    \sum_{t=1}^T \left(\alpha - u_{\tilde{m},t}\right) \leq Q_{T+1} - Q_1 = Q_{T+1},
\end{align}
so that
\begin{align}
    \frac{1}{T}\sum_{t=1}^T u_{\tilde{m},t} \geq \alpha - \frac{Q_{T+1}}{T}.
    \label{eq:proof_corollary_u_lower}
\end{align}

Let $Y_t\sim\mathrm{Bernoulli}(\feedbackprob)$ denote the feedback-availability indicator at step $t$, which is independent of the routing decision and the satisfaction outcomes under \Cref{assumption:partial_feedback}. By the definition of $u_{\tilde{m},t}$ in \eqref{eq:queue_update_u_cases}, we have $u_{\tilde{m},t} = Y_t\, s_{\tilde{m},t} + (1-Y_t)\,\hat{s}_{\tilde{m},t}$. Taking expectation and using the independence of $Y_t$, we have
\begin{align}
    \Expectbracket{u_{\tilde{m},t}}
    = \feedbackprob\,\Expectbracket{s_{\tilde{m},t}} + (1-\feedbackprob)\Expectbracket{\hat{s}_{\tilde{m},t}}
    = \Expectbracket{s_{\tilde{m},t}} + (1-\feedbackprob)\Expectbracket{\hat{s}_{\tilde{m},t} - s_{\tilde{m},t}}.
\end{align}
Since $\hat{s}_{\tilde{m},t} - s_{\tilde{m},t} \leq |\hat{s}_{\tilde{m},t} - s_{\tilde{m},t}| \leq \max_m |\hat{s}_{m,t} - s_{m,t}|$, the expectation version of \Cref{lemma:satisfaction_error_bound_with_delta} gives
\begin{align}
    \Expectbracket{u_{\tilde{m},t}} \leq \Expectbracket{s_{\tilde{m},t}} + (1-\feedbackprob)\left(\beta F_\mathrm{min} + \mathcal{O}\left(\frac{\beta (1+\sigma^2)}{t^{\sfrac{3}{4}}\feedbackprob}\right)\right).
    \label{eq:proof_corollary_u_upper}
\end{align}
Because only one model serves request $t$, we have $\sum_{m=1}^M y_{m,t} s_{m,t} = s_{\tilde{m},t}$. Combining \eqref{eq:proof_corollary_u_lower} and \eqref{eq:proof_corollary_u_upper} after taking total expectation, and noting that $\frac{1}{T}\sum_{t=1}^T t^{-\frac{3}{4}} = \mathcal{O}(T^{-\frac{3}{4}})$, we obtain
\begin{align}
    \alpha - \frac{1}{T}\sum_{t=1}^T \sum_{m=1}^M \Expectbracket{y_{m,t} s_{m,t}}
    \leq \frac{\Expectbracket{Q_{T+1}}}{T} + (1-\feedbackprob)\beta F_\mathrm{min}
    + \mathcal{O}\left(\frac{(1-\feedbackprob)\beta (1+\sigma^2)}{\feedbackprob\, T^{\sfrac{3}{4}}}\right).
\end{align}
Finally, applying the queue-length bound $\Expectbracket{Q_{T+1}} \leq \gamma + \mathcal{O}(\sqrt{T})$ from \Cref{theorem:max_queue_length} gives $\frac{\Expectbracket{Q_{T+1}}}{T} \leq \frac{\gamma}{T} + \mathcal{O}\left(\frac{1}{\sqrt{T}}\right)$. Since $T^{-\frac{3}{4}} \leq \mathcal{O}(T^{-\frac{1}{2}})$, the last term is absorbed, yielding
\begin{align}
    \alpha - \frac{1}{T}\sum_{t=1}^T \sum_{m=1}^M \Expectbracket{y_{m,t} s_{m,t}}
    \leq \mathcal{O}\left(\frac{\gamma}{T} + \frac{1}{\sqrt{T}}\right) + (1-\feedbackprob)\beta F_\mathrm{min}.
\end{align}

\subsection{Proof of \Cref{theorem:cost_optimality_bound}}

We first consider any $t$ such that $X_t = 0$, i.e., no exploration or predictor training.
As before, let $\tilde{m}$ denote the solution to \eqref{eq:per-request-optimization} for a given $t$, i.e., $y_{\tilde{m},t}=1$ and $y_{m',t}=0$ for $m'\neq \tilde{m}$. We also use the definition in \eqref{eq:queue_update_u_cases}.
We consider the following Lyapunov drift of the queue length:
\begin{align}
    \frac{1}{2}\Expectcond{Q^2_{t+1} - Q^2_t}{Q_t}
    & \leq \frac{1}{2} \Expectcond{\left(Q_t + \alpha - u_{\tilde{m},t}\right)^2 - Q^2_t}{Q_t} \nonumber\\
    & \leq \frac{1}{2} \Expectcond{2Q_t (\alpha - u_{\tilde{m},t}) + 1}{Q_t} \nonumber\\
    & \leq \Expectcond{Q_t (\alpha - \hat{s}_{\tilde{m},t})}{Q_t}
      + Q_t \Expectbracket{\max_m|\hat{s}_{m,t} - s_{m,t}|} + \frac{1}{2},
\label{eq:proof_optimality_drift}
\end{align}
where the second inequality uses $(\alpha - u_{\tilde{m},t})^2 \leq 1$, and the third adds and subtracts $\hat{s}_{\tilde{m},t}$, considers that $u_{\tilde{m},t}$ is either $s_{\tilde{m},t}$ or $\hat{s}_{\tilde{m},t}$, and bounds the resulting estimation error term.

Let $\{y_{m,t}^{\mathrm{OPT}}, \forall m\}$ denote an optimal stationary policy for \eqref{eq:original-problem}.
Since $\tilde{m}$ minimizes $V C_{m,t} + Q_t(\alpha - \hat{s}_{m,t})$ over all models, we have
\begin{align}
    & \Expectcond{V C_{\tilde{m},t} + Q_t\left(\alpha - \hat{s}_{\tilde{m},t}\right)}{Q_t} \nonumber\\
    & \leq V \Expectbracket{\sum_{m=1}^M y_{m,t}^{\mathrm{OPT}} C_{m,t}}
      + Q_t \Expectbracket{\sum_{m=1}^M y_{m,t}^{\mathrm{OPT}}(s_{m,t} - \hat{s}_{m,t})}
      + Q_t \Expectbracket{\sum_{m=1}^M y_{m,t}^{\mathrm{OPT}}(\alpha - s_{m,t})} \nonumber\\
    & \leq V C^{\mathrm{OPT}} + Q_t \Expectbracket{\max_m|\hat{s}_{m,t} - s_{m,t}|},
    \label{eq:proof_drift_plus_penalty}
\end{align}
where the last inequality uses the fact that the optimal stationary policy satisfies \eqref{eq:constraint} with equality, so $\Expectbracket{\sum_m y_{m,t}^{\mathrm{OPT}}(\alpha - s_{m,t})} = 0$.

Combining \eqref{eq:proof_optimality_drift} and \eqref{eq:proof_drift_plus_penalty}, canceling out $\Expectcond{Q_t (\alpha - \hat{s}_{\tilde{m},t})}{Q_t}$ on both sides, and taking total expectation gives
\begin{align*}
    \frac{1}{2}\Expectbracket{Q^2_{t+1} - Q^2_t} + V\Expectbracket{\sum_{m=1}^M y_{m,t} C_{m,t}}
    \leq V C^{\mathrm{OPT}} + 2\Expectbracket{Q_t} \cdot \Expectbracket{\max_m|\hat{s}_{m,t} - s_{m,t}|} + \frac{1}{2}.
\end{align*}

Rearranging and applying \Cref{lemma:satisfaction_error_bound_with_delta} yields
\begin{align*}
    \Expectbracket{\sum_{m=1}^M y_{m,t} C_{m,t}}
    &\leq C^{\mathrm{OPT}}
      \!+\! \frac{2}{V}\Expectbracket{Q_t}\!\cdot\!\left(\beta F_\mathrm{min} 
        \!+\! \mathcal{O}\left(\frac{\beta (1\!+\!\sigma^2)}{ t^{\sfrac{3}{4}} \feedbackprob}\right)\right)
      \!+\! \frac{1}{2V} \!+\! \frac{1}{2V}\Expectbracket{Q^2_t \!-\! Q^2_{t+1}}.
\end{align*}

Applying \Cref{theorem:max_queue_length} to bound $\Expectbracket{Q_t} \leq \gamma + \mathcal{O}\left(\sqrt{t}\right)$ and absorbing $1+\sigma^2$ into $\mathcal{O}(\cdot)$, we obtain
\begin{align*}
    \Expectbracket{\sum_{m=1}^M y_{m,t} C_{m,t}}
    \leq C^{\mathrm{OPT}}
    \!+\! \frac{2}{V}\left(\mathcal{O}\!\left((\gamma \!+\! \sqrt{t})\,\frac{\beta}{t^{3/4}}\right)
    \!+\! \beta F_{\mathrm{min}} \Expectbracket{Q_t}\right)
    \!+\! \frac{1}{2V} \!+\! \frac{1}{2V}\Expectbracket{Q^2_t \!-\! Q^2_{t+1}}.
\end{align*}

For any $t$ with $X_t = 1$, the cost incurred during exploration is at most $C_{\max}$.
Averaging over all $T$ requests and telescoping the queue term, we obtain
\begin{align*}
    &\Expectbracket{\frac{1}{T}\sum_{t=1}^T \sum_{m=1}^M y_{m,t} C_{m,t}} \\
    &\leq C^{\mathrm{OPT}}
      + \frac{2}{V}\left(\frac{1}{T}\sum_{t=1}^T \mathcal{O}\!\left((\gamma+\sqrt{t})\,\frac{\beta}{t^{\frac{3}{4}}}\right)
      + \frac{\beta F_{\mathrm{min}}}{T}\sum_{t=1}^T \Expectbracket{Q_t}\right)
      + \frac{1}{2V} + \frac{\Expectbracket{K} C_{\max}}{T} \\
    &\leq C^{\mathrm{OPT}}
      + \mathcal{O}\!\left(\frac{1}{T}\sum_{t=1}^T \frac{\beta}{t^{\frac{1}{4}}} + \beta F_{\mathrm{min}} + \frac{1}{V} + \frac{C_{\max}}{T^{\frac{1}{4}}}\right) \\
    &= C^{\mathrm{OPT}} + \mathcal{O}\!\left(\frac{\beta}{\sqrt[4]{T}} + \beta F_{\mathrm{min}} + \frac{1}{V}\right),
\end{align*}
where the second inequality applies \Cref{theorem:max_queue_length} together with $\Expectbracket{K} = \sum_{t=1}^T p_t = \Theta(T^{\sfrac{3}{4}})$, which follows from the same computation as in the proof of \Cref{lemma:k_bound} but without the feedback factor $\feedbackprob$, because exploration cost is incurred regardless of whether feedback is received or not (we only know whether there is feedback after the model is chosen), and the final equality absorbs the fixed cost scale $C_{\max}$ into $\mathcal{O}(\cdot)$.

\clearpage
\section{Experimental Details \& Additional Results}
\label{app:reproducibility}

We provide a ready-to-run code base on GitHub: 
\emph{Please reach out for repo access while the preprint is under review.}
This section provides additional details to reproduce our results and additional experimental results.

\subsection{Environment and Hardware}
\label{app:env}

All experiments were executed in a single Python 3.12 environment managed with \texttt{uv}. 
A full list of locked dependencies can be found in the code base.

\textbf{Random seeding.} We seed all random generators in our experiment, usually with seeds 42, 1, and 1234, if not specified otherwise. 
For experiment where we use more than 3 seeds, we use the range $[1, n]$.

\textbf{Hardware.} 
We conducted all our experiments on a SLURM cluster equipped with NVIDIA H100 GPUs.
To collect the benchmark data, we used 1 GPU for Qwen 3.5 2B, 9B-AWQ, and 35B-AWQ. 
We used 2 GPUs to run Qwen 3.5 122B-AWQ.
Otherwise, all experiments presented in this paper require little GPU memory since the routing models are typically $\ll1\mathrm{B}$ parameters.

\subsection{Dataset}
\label{app:dataset}

We collect all benchmark data by using EleutherAI's \texttt{lm-evaluation-harness}. 
We then merge the logged json outputs into a single parquet file that can be consumed by our code base. 
The dataset contains the cost for querying all LLMs for a specific request (in Megajoules) as well as the actual model response. 
The responses are particularly relevant for the experiments with the LLM judge.
We measure energy using \texttt{zeus-ml}.
Each row in the dataset corresponds to one evaluation example and contains:

\begin{itemize}
    \item Identifiers: \texttt{doc\_id}, \texttt{benchmark}, \texttt{input\_text}, \texttt{correct\_answer}
    \item Per-model binary outcomes: \texttt{\{model\}\_solved}
    \item Per-model energy/latency: \texttt{\{model\}\_energy\_joules}, \texttt{\{model\}\_latency\_ms}
    \item Per-model token counts: \texttt{\{model\}\_prompt\_tokens}, \texttt{\{model\}\_completion\_tokens}
\end{itemize}

\paragraph{Model zoo.} We use one model zoo with 4 distinct models, specifically \texttt{Qwen/Qwen3.5-2B}, \texttt{QuantTrio/Qwen3.5-9B-AWQ}, \texttt{QuantTrio/Qwen3.5-35B-A3B-AWQ}, and \texttt{QuantTrio/Qwen3.5-122B-A10B-AWQ}. Further characteristics are available in \Cref{tab:appendix_model_zoo}.

\begin{table*}[h]
\centering
\caption{Model zoo configuration. All models are from the Qwen 3.5 family. Energy costs are per million tokens (MJ).}
\label{tab:appendix_model_zoo}
\setlength{\tabcolsep}{6pt}
\renewcommand{\arraystretch}{1.1}
\resizebox{0.85\textwidth}{!}{
    \begin{tabular}{l c c c}
    \toprule
    \textbf{Model} & \textbf{Parameters (B)} & \textbf{Category} & \textbf{Avg. Cost across 11 benchmarks (MJ/1M tok)} \\
    \midrule
    Qwen3.5-2B & 2.0 & small & 0.22 \\
    Qwen3.5-9B-AWQ & 9.0 & medium & 0.32 \\
    Qwen3.5-35B-A3B-AWQ & 35.0 & xlarge & 0.53 \\
    Qwen3.5-122B-A10B-AWQ & 122.0 & xxlarge & 1.16\\
    \bottomrule
    \end{tabular}
}
\end{table*}

\paragraph{Model performance.} Figure~\ref{fig:appendix_model_performance} shows per-model accuracy across all 11 benchmarks. Each group of bars corresponds to one benchmark, with individual bars showing the accuracy of each Qwen 3.5 model.

\begin{figure*}[!t]
    \centering
    \includegraphics[width=\textwidth]{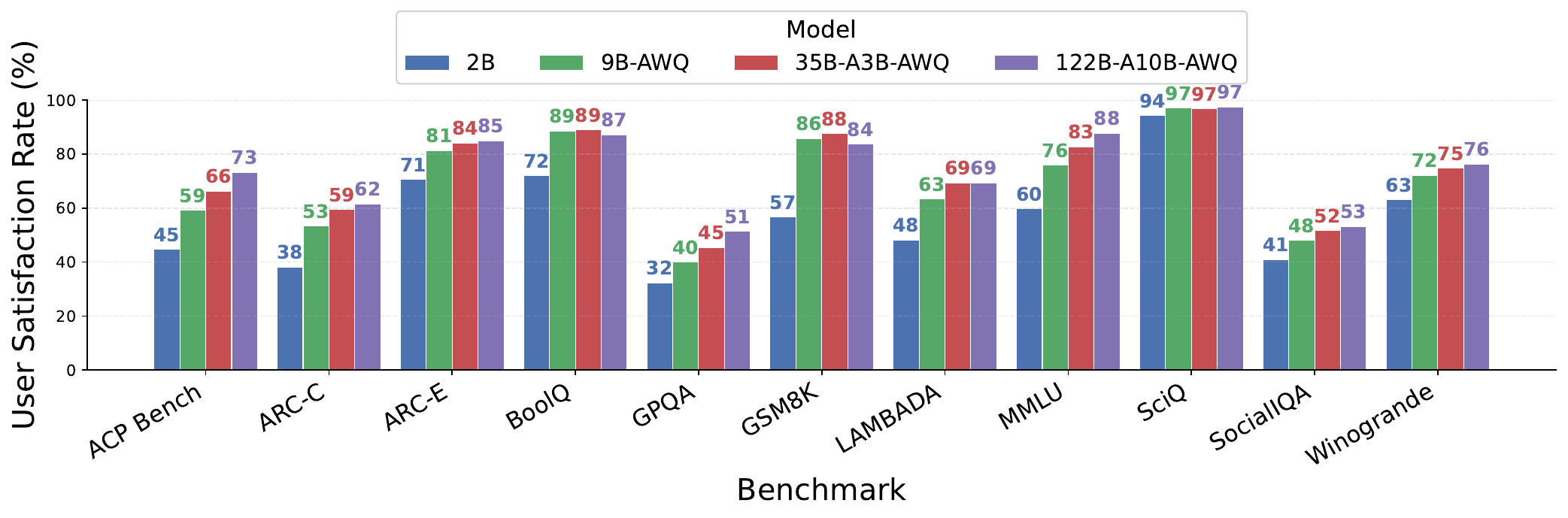}
    \caption{Per-model accuracy across all 11 benchmarks. Each group of bars corresponds to one benchmark, with individual bars showing the accuracy of each Qwen 3.5 model. Accuracy values are printed on top of each bar.}
    \label{fig:appendix_model_performance}
\end{figure*}

\subsection{Predictor: Architecture and Training Details}
\label{app:predictor_training}

\textbf{Architecture.}
The satisfaction predictor is a multi-label binary classifier built on top of a frozen
ModernBERT encoder ($D=768$ dimensions).
The classification head is a two-layer MLP:
$\texttt{Dropout}(p=0.1) \rightarrow \texttt{Linear}(D, D)
\rightarrow \texttt{LayerNorm} \rightarrow \texttt{ReLU}
\rightarrow \texttt{Dropout}(p=0.1) \rightarrow \texttt{Linear}(D, |M|)$,
where $M$ is the number of candidate models.
At inference time the model outputs $M$ logits, one per model head,
which are passed through a sigmoid to obtain $\hat{s}_m \in [0,1]^M$.

\textbf{Optimizer and training hyperparameters.}
Only the classifier-head parameters are trained; the encoder remains frozen.
Optimization uses SGD with momentum $0.90$ and weight decay $0.01$.
The default learning rate is $\eta = 0.006$ (identified via extensive hyperparameter tuning).
Gradients are clipped to a maximum $\ell_2$-norm of $1.0$.

\textbf{Constant versus diminishing learning rate.}
Our analysis (\Cref{lemma:sgd_bound}) uses the diminishing schedule $\eta_k=\min\{\sfrac{4}{\mu(k+1)},\,\sfrac{1}{2L}\}$, which is the standard technique for last-iterate convergence to the exact minimizer $F_\mathrm{min}$. In practice we found a constant rate already sufficient for good performance, so we use it for simplicity. The only cost is convergence to an $\mathcal{O}(\eta\sigma^2)$ neighborhood of $F_\mathrm{min}$ instead of $F_\mathrm{min}$ exactly, which inflates the floor $\beta F_\mathrm{min}$ by a constant and leaves all guarantees intact (they already carry this floor).

\textbf{Loss function.}
Training minimizes a masked binary cross-entropy loss.
Because each exploration step reveals feedback for only one model,
the label matrix contains NaN for all unobserved heads;
the loss back-propagates only through the observed entries.

\textbf{Minibatch buffer.}
Training occurs only when the minibatch buffer reaches capacity
$\texttt{batch\_size}=16$.
The buffer stores tuples $(\text{text}, m, \text{satisfaction}).$
We sample minibatches randomly from the buffer for every predictor update. 
When full, the buffer yields a label matrix of shape $(B, M)$ with exactly
one label per request.
By default a single SGD step is taken per full minibatch
($\texttt{classifier\_epochs}=1$); setting this $>1$ enables multiple passes.

\textbf{Default hyperparameters (full list).}
$\eta=0.006$, momentum$=0.90$, weight decay$=0.01$,
batch size$=16$, dropout$=0.1$, queue mode$=\texttt{random}$,
explore mode$=\texttt{uniform}$, classifier epochs$=1$,
$\texttt{lr\_schedule}=\texttt{fixed}$.

\subsection{Performance Matching Between Baselines and \algname}
We carefully tune the cost parameters of CARROT and Causal Router to match the performance requirements we set for each benchmarks. 
For CARROT, we search the range $\mu \in [0, 1]$ with increments of 0.05 and for Causal Router we search $\lambda \in [0, 0.01]$ in increments of $0.001$.\footnote{Note, our $\lambda$ range is significantly different from the original paper as we are using energy consumption in joules while the original paper uses USD per 1M tokens.} 
We match the performance against the SLA requirement $\alpha$ and pick the configuration that comes closest to $\alpha$. 
In order to ensure there is enough training data available for the baselines, we chose \texttt{--target-feedback-rate 0.2}. 
To train the CARROT and Causal Router classifiers, we use a separate training dataset that consists of the benchmark trainings splits. 
We collect the data in the same way per our procedure described above for the actual benchmark datasets.
We train the classifiers with the same amount of feedback as what MESS+ and \algname see at inference time to ensure fair comparison. 
Full training scripts are provided in the code base for your reference.
The classifier training statistics are provided in \Cref{tab:baseline_stats}.
The detailed parameterization is provided in \Cref{tab:matched_hyperparameters}. 

\begin{table}[t]
\centering
\caption{Validation statistics for offline-trained baselines (feedback rate 0.2).}
\label{tab:baseline_stats}
\resizebox{0.7\textwidth}{!}{
    \begin{tabular}{@{}lccccc@{}}
    \toprule
    Baseline & Accuracy & Precision & Recall & F1 & ROC-AUC \\
    \midrule
    Causal RM-Softmax & 0.5833 & 0.3362 & 0.3550 & 0.2763 & nan \\
    CARROT RoBERTa & 0.3400 & 0.6596 & 0.9121 & 0.7625 & 0.6593 \\
    \bottomrule
    \end{tabular}
}
\end{table}

\begin{table}[t]
    \centering
    \small
    \caption{Matched cost sensitivity hyperparameters for CARROT ($\mu$) and Causal rm\_interval ($\lambda$) routers, tuned to meet SLARouter target accuracy $\alpha$ per benchmark.}
    \label{tab:matched_hyperparameters}
    \renewcommand{\arraystretch}{1.0}
    \resizebox{0.5\textwidth}{!}{
        \begin{tabular}{l c c}
            \toprule
            \textbf{Benchmark} & $\boldsymbol{\mu}$ (CARROT) & $\boldsymbol{\lambda}$ (Causal) \\
            \midrule
            ACP Bench ($\alpha=0.65$) & 0.5 & 0.000 \\
            ARC-Challenge ($\alpha=0.52$) & 0.8 & 0.001 \\
            ARC-Easy ($\alpha=0.80$) & 0.8 & 0.001 \\
            BoolQ ($\alpha=0.85$) & 0.7 & 0.001 \\
            GPQA ($\alpha=0.45$) & 0.3 & 0.000 \\
            GSM8K ($\alpha=0.83$) & 0.6 & 0.000 \\
            LAMBADA ($\alpha=0.65$) & 0.6 & 0.001 \\
            MMLU ($\alpha=0.75$) & 0.7 & 0.000 \\
            SciQ ($\alpha=0.96$) & 0.8 & 0.001 \\
            SocialIQa ($\alpha=0.48$) & 0.8 & 0.005 \\
            WinoGrande ($\alpha=0.73$) & 0.7 & 0.000 \\
            \bottomrule
        \end{tabular}
    }
\end{table}

\subsection{Data-Driven Selection of $V$}
\label{sec:appendix_data_driven_v_selection}
We compare a grid search for an appropriate $V$ value with our data-driven selection mechanism (\Cref{tab:updated_appendix_v_sweep}).
We note that per our theoretical analysis, we will always reach $\alpha$ and $V$ only determines how long it takes for \algname to converge. 
Yet, when choosing $V$ too large, convergence times may become impractical.
This is what can be seen when choosing $V > 0.0001$.
These values are one order of magnitude apart from one another, so making an educated guess at $V$ based on reference benchmarks may yield unnecessary inefficiencies. 
So, overall our data-driven method increases the cost effectiveness of \algname beyond what would be possible with a grid search.

\subsection{Judge Feedback Alignment}
\label{sec:appendix_judge_feedback_alignment}

When user feedback is unavailable, we use an LLM-based judge as a scalable alternative to obtain synthetic feedback signals.
Specifically, we prompt GPT-5.4 to assess the input/response pairs, and collect binary (correct/incorrect) judgment scores.
We measure the alignment (agreement rate) between the judge's outputs and ground-truth benchmark labels across two challenging evaluations: GPQA~\cite{gpqa}, a graduate-level science benchmark requiring domain expertise, and ACPBench~\cite{acpbench}, a planning benchmark that demands multi-step reasoning. We see that disabling reasoning leads to poor alignment (Table~\ref{tab:results}) and substantially higher error rates across both benchmarks. This confirms the assumption that the SOTA reasoning-capable LLMs are necessary to reliably replicate the user feedback. In our experiments we use the latter setting with feedback collected from a reasoning-enabled judge.
Our prompt template to obtain feedback from the judge is geared towards receiving binary feedback to serve as a plug-and-play substitute for actual user feedback~(\Cref{lst:prompt_template}).

\begin{lstlisting}[caption={GPT-5.4 judge prompt template}, label={lst:prompt_template}]
    ---SYSTEM---
    You are an expert evaluator for AI benchmark answers.
    
    {%
    This question has a fixed set of labeled answer choices - the model must select exactly one.
    
    Follow this process in order:
    1. Read the question and all answer choices.
    2. Independently determine which choice is correct - do this BEFORE reading the model's response.
    3. Compare: does the model's final answer match the correct choice?
    
    {%
    This question requires a specific answer (a number, yes/no, word, or sentence completion) - there are no labeled choices.
    
    Follow this process in order:
    1. Read the question and any provided passage or context.
    2. Independently determine the correct answer using your knowledge and reasoning - do this BEFORE reading the model's response.
       - Math: solve step by step and identify the final numerical answer.
       - Yes/No: derive from the passage or your knowledge.
       - Word/phrase prediction: identify the correct word or phrase from context.
       - Fill-in-the-blank: determine which completion is logically and grammatically correct.
    3. Compare: does the model's answer match the correct answer?
       - Numbers: treat equivalent forms as equal ("18", "\$18", "18.0" all match 18).
       - Yes/No: treat "True"/"1" as yes and "False"/"0" as no.
       - Words/phrases: accept exact match or an unambiguous paraphrase.
    {%
    
    Scoring:
    - 1 (Correct): Model's final answer matches the correct answer.
    - 0 (Incorrect): Model's answer is wrong, ambiguous, or the model failed to commit.
    
    ---USER---
    {%
    Think through the correct answer first, then evaluate the model.
    {%
    
    Question:
    {{ input_text }}
    
    {%
    Answer choices:
    {{ choices }}
    {%
    
    Model response to evaluate:
    {{ response }}
    
    Respond ONLY with valid JSON:
    {"label": <0 or 1>}
\end{lstlisting}

\begin{table*}[h]
\centering
\caption{%
  Full V sweep for SLARouter on ACPBench and GPQA.
  \textcolor{darkgreen}{Green} marks configurations satisfying the service
  level~$\alpha$; \textcolor{darkred}{red} marks violations.
}
\label{tab:updated_appendix_v_sweep}
\setlength{\tabcolsep}{6pt}
\renewcommand{\arraystretch}{0.95}
\resizebox{\textwidth}{!}{%
\begin{tabular}{l l ccc ccc}
\toprule
\multicolumn{2}{c}{} & \multicolumn{3}{c}{\textbf{ACP Bench} ($\alpha$=65\%)} & \multicolumn{3}{c}{\textbf{GPQA} ($\alpha$=45\%)} \\
\cmidrule(lr){3-5}\cmidrule(lr){6-8}
\textbf{Variant} & \textbf{$V$ value} & \thead{Operating\\Cost (MJ)} & \thead{Post-Conv.\\Accuracy} & \thead{Model Call Ratio\\(2B/9B/35B/122B)} & \thead{Operating\\Cost (MJ)} & \thead{Post-Conv.\\Accuracy} & \thead{Model Call Ratio\\(2B/9B/35B/122B)} \\
\cmidrule(lr){1-2}\cmidrule(lr){3-5}\cmidrule(lr){6-8}
\multirow{7}{*}{Fixed $V$}
& $1\times10^{-6}$ & $2.796_{\pm0.147}$ & \textcolor{darkgreen}{$66.5_{\pm1.0}$} & 10\%/14\%/60\%/16\% & $0.068_{\pm0.014}$ & \textcolor{darkgreen}{$46.0_{\pm1.0}$} & 2\%/0\%/58\%/40\% \\
& $1\times10^{-5}$ & $2.786_{\pm0.164}$ & \textcolor{darkgreen}{$65.6_{\pm0.2}$} & 10\%/16\%/58\%/16\% & $0.065_{\pm0.013}$ & \textcolor{darkgreen}{$45.8_{\pm0.4}$} & 3\%/2\%/58\%/37\% \\
& $3\times10^{-5}$ & $2.712_{\pm0.124}$ & \textcolor{darkgreen}{$65.1_{\pm0.2}$} & 10\%/20\%/58\%/12\% & $0.063_{\pm0.008}$ & \textcolor{darkgreen}{$45.1_{\pm0.2}$} & 5\%/2\%/56\%/37\% \\
& \textbf{$1\times10^{-4}$} & \textbf{$2.489_{\pm0.336}$} & \textbf{\textcolor{darkred}{$64.9_{\pm0.3}$}} & 14\%/32\%/45\%/9\% & $0.058_{\pm0.000}$ & \textcolor{darkred}{$44.3_{\pm1.6}$} & 11\%/4\%/56\%/29\% \\
& \textbf{$3\times10^{-4}$} & $2.330_{\pm0.255}$ & \textcolor{darkred}{$64.7_{\pm0.8}$} & 22\%/38\%/33\%/7\% & \textbf{$0.058_{\pm0.001}$} & \textbf{\textcolor{darkred}{$44.2_{\pm1.7}$}} & 11\%/4\%/56\%/29\% \\
& $1\times10^{-3}$ & $2.151_{\pm0.156}$ & \textcolor{darkred}{$63.5_{\pm1.3}$} & 23\%/40\%/32\%/6\% & $0.045_{\pm0.001}$ & \textcolor{darkred}{$43.4_{\pm1.7}$} & 17\%/9\%/62\%/12\% \\
& $3\times10^{-3}$ & $2.091_{\pm0.042}$ & \textcolor{darkred}{$63.0_{\pm1.7}$} & 23\%/40\%/32\%/6\% & $0.044_{\pm0.002}$ & \textcolor{darkred}{$43.1_{\pm3.0}$} & 18\%/9\%/61\%/12\% \\
\midrule
Data-driven & $0.000029$ / $0.000035$ & $2.641_{\pm0.200}$ & \textcolor{darkgreen}{$65.1_{\pm0.3}$} & 10\%/20\%/58\%/12\% & $0.060_{\pm0.007}$ & \textcolor{darkgreen}{$45.1_{\pm0.1}$} & 4\%/3\%/56\%/37\% \\
\bottomrule
\end{tabular}%
}
\end{table*}

\begin{table*}[!h]
\centering
\caption{Agreement rate between the judge (GPT-5.4) and the user feedback (ground truth labels).}
\label{tab:results}
\setlength{\tabcolsep}{6pt}
\renewcommand{\arraystretch}{1.1}
\resizebox{0.85\textwidth}{!}{
    \begin{tabular}{l c c c c}
    \toprule
    \textbf{Benchmark} & \textbf{Reasoning Enabled} & \textbf{Agreement Rate} & \textbf{False Positive Rate} & \textbf{False Negative Rate} \\
    \midrule
    ACPBench           & \xmark & 68\% & 26.3\% & 35.5\% \\
    GPQA & \xmark & 62\% & 32.1\% & 45.5\% \\
    \midrule
    ACPBench         & \cmark & 94.4\% & 6.6\%  & 4.9\%  \\
    GPQA               & \cmark & 87.9\% & 12.9\% & 10.9\% \\
    \bottomrule
    \end{tabular}
}
\vspace{-3em}
\end{table*}


\end{document}